\def\ColtFormat{1}

\ifnum\ColtFormat=1
\documentclass[final,12pt]{colt2023} %
\usepackage{etoolbox}
\newtoggle{coltformat}
\toggletrue{coltformat} 
\newcommand{\colt}[1]{\iftoggle{coltformat}{#1}{}} 
\newcommand{\arxiv}[1]{\iftoggle{coltformat}{}{#1}}
\else
\documentclass[11pt]{article} 
\usepackage{etoolbox}
\newtoggle{coltformat}
\togglefalse{coltformat} 
\newcommand{\colt}[1]{\iftoggle{coltformat}{#1}{}}
\newcommand{\arxiv}[1]{\iftoggle{coltformat}{}{#1}}

\(\)
\arxiv{\usepackage[letterpaper, left=1in, right=1in, top=1in, bottom=1in]{geometry}

\usepackage{xcolor} 
\usepackage[colorlinks=true, linkcolor=blue, citecolor=blue]{hyperref}
\usepackage{microtype}

\usepackage{natbib}
\bibliographystyle{plainnat}
\bibpunct{(}{)}{;}{a}{,}{,}

\usepackage{amsthm}
\usepackage{mathtools}
\usepackage{amsmath}
\usepackage{bbm}
\usepackage{amsfonts}
\usepackage{amssymb}
\let\vec\undefined
\usepackage{MnSymbol} %

\usepackage{xpatch}

\theoremstyle{definition}  %
\newtheorem{assumption}{Assumption}
\newtheorem{claim}{Claim}

\newtheorem{proposition}{Proposition}
\crefname{proposition}{Proposition}{Propositions}

\theoremstyle{plain}
\newtheorem{lemma}{Lemma}
\newtheorem{remark}{Remark}

\newtheorem{theorem}{Theorem}
\newtheorem{definition}{Definition}

\xpatchcmd{\proof}{\itshape}{\normalfont\proofnameformat}{}{}
\newcommand{\proofnameformat}{\bfseries}

  \newcommand{\jmlrBlackBox}{\rule{1.5ex}{1.5ex}}
  \newcommand{\jmlrQED}{\hfill\jmlrBlackBox\par\bigskip}

}
\usepackage{xspace}
\fi

\colt{\title[Ticketed Learning--Unlearning Schemes]{Ticketed Learning--Unlearning Schemes}}
\arxiv{\title{Ticketed Learning--Unlearning Schemes\thanks{Authors are sorted alphabetically by their last name.}}}

\usepackage{times}

\colt{\coltauthor{%
 \Name{Badih Ghazi} \Email{badihghazi@gmail.com}\\
 \Name{Pritish Kamath} \Email{pritish@alum.mit.edu}\\
 \Name{Ravi Kumar} \Email{ravi.k53@gmail.com}\\
 \addr Google Research, Mountain View%
 \AND
 \Name{Pasin Manurangsi} \Email{pasin@google.com}\\ 
 \addr Google Research, Thailand
 \AND
 \Name{Ayush Sekhari} \Email{sekhari@mit.edu}\\
 \addr Massachusetts Institute of Technology%
 \AND
 \Name{Chiyuan Zhang} \Email{chiyuan@google.com}\\
 \addr Google Research, Mountain View
}}

\arxiv{
\date{}
\author{ Badih Ghazi\thanks{Google Research, Mountain View}\\
        {\small\texttt{badihghazi@gmail.com}}\\ 
	  \and
	  Pritish Kamath\footnotemark[2] \\ 
	  {\small\texttt{pritishk@google.com}}\\
	  \and
	  Ravi Kumar\footnotemark[2]\\
	  {\small\texttt{ravi.k53@gmail.com}}\\
          \and
 Pasin Manurangsi\thanks{Google Research, Thailand}\\
        {\small\texttt{pasin@google.com}}\\ 
	  \and
	Ayush Sekhari\thanks{Massachusetts Institute of Technology}\\
                {\small\texttt{sekhari@mit.edu}}\\
       \and
        Chiyuan Zhang\footnotemark[2] \\  
        {\small\texttt{chiyuan@google.com}}\\
}
}

\def\Comments{0} %
\usepackage{amsmath,amssymb,mathtools}
\usepackage{bm}
\usepackage{dsfont}

\usepackage{enumitem}

\setlist{nosep}

\usepackage{tikz}
\usetikzlibrary{arrows.meta,calc,trees}

\colt{\usepackage{hyperref}
\hypersetup{
    colorlinks=true,
    linkcolor=blue,
    citecolor=Blue2,
    filecolor=Maroon0,
    urlcolor=Maroon0
}}
\arxiv{\usepackage[colorlinks=true, linkcolor=blue, citecolor=blue]{hyperref}}

\usepackage[nameinlink]{cleveref}
\Crefname{lemma}{Lemma}{Lemmas}
\crefname{lemma}{Lemma}{Lemmas}

\definecolor{Gred}{RGB}{219, 50, 54}
\definecolor{Ggreen}{RGB}{60, 186, 84}
\definecolor{Gblue}{RGB}{72, 133, 237}
\definecolor{Gyellow}{RGB}{247, 178, 16}
\definecolor{ToCgreen}{RGB}{0, 128, 0}
\definecolor{myGold}{RGB}{231,141,20}
\definecolor{myBlue}{rgb}{0.19,0.41,.65}
\definecolor{myPurple}{RGB}{175,0,124}

\providecommand{\Comments}{0}
\ifnum\Comments=1
\usepackage[colorinlistoftodos,prependcaption,textsize=scriptsize]{todonotes}
\paperwidth=\dimexpr \paperwidth + 2.7cm\relax
\evensidemargin=\dimexpr\evensidemargin + 1.8cm\relax
\else
\usepackage[disable]{todonotes}
\fi
\newcommand{\mytodo}[1]{\ifnum\Comments=1{#1}\fi}

\newcommand{\tableoftodos}{\ifnum\Comments=1 \listoftodos[Comments/To Do's] \fi}

\newcommand{\poly}{\mathrm{poly}}
\newcommand{\defeq}{\triangleq}

\newcommand{\indicator}[1]{\mathds{1}\crl{#1}}

\newcommand{\bit}{\crl{0,1}}

\newcommand{\Learn}{\mathsf{Learn}\xspace}
\newcommand{\Unlearn}{\mathsf{Unlearn}\xspace}
\newcommand{\ERM}{\mathsf{ERM}\xspace}
\newcommand{\aux}{\mathsf{aux}\xspace}
\newcommand{\Encode}{\mathsf{Encode}\xspace}
\newcommand{\Decode}{\mathsf{Decode}\xspace}
\newcommand{\Merge}{\mathsf{Merge}\xspace}

\newcommand{\Compress}{\mathsf{Compress}\xspace}

\newcommand{\thresh}{\mathrm{th}}
\newcommand{\point}{\mathrm{pt}}

\newcommand{\Hth}{\cH_{\thresh}}
\newcommand{\Hdth}{\cH^d_{\thresh}}
\newcommand{\Hpt}{\cH_{\point}}

\newcommand{\N}{\mathbb{N}}

\newcommand{\cQ}{\mathcal{Q}}

\newcommand{\tLearn}{\widetilde{\Learn}}
\newcommand{\tUnlearn}{\widetilde{\Unlearn}}
\newcommand{\minval}{\textsc{MinVal}\xspace}
\newcommand{\maxval}{\textsc{MaxVal}\xspace} 
\newcommand{\succe}{\mathrm{succ}}
\newcommand{\ba}{{\bm a}}

\newcommand{\hstar}{h^\star} 
\newcommand{\indic}[1]{\mathds{1}\crl{#1}}
\newcommand{\err}{\mathrm{err}}

\newcommand{\U}{{S_I}}

\colt{
	\setlist[itemize]{leftmargin=6mm}
}

\colt{
	\newenvironment{proofof}[1][\unskip]{%
		\par\medskip\noindent{\bfseries\upshape Proof of #1 \/ }%
	}{\jmlrQED}
}

\arxiv{
	\newenvironment{proofof}[1][\unskip]{%
		\par\medskip\noindent{\bfseries\upshape Proof of #1 \/ }%
	}{\qed}
} %

\usepackage{enumitem}
\usepackage{nicefrac}

\usepackage{amsxtra}

\usepackage{enumitem}

\allowdisplaybreaks

\DeclarePairedDelimiter{\abs}{\lvert}{\rvert} 

\DeclarePairedDelimiter{\crl}{\{}{\}}
\DeclarePairedDelimiter{\prn}{(}{)}

\DeclarePairedDelimiter{\tri}{\langle}{\rangle}

\let\Pr\undefined

\DeclareMathOperator{\Pr}{Pr}

\DeclareMathOperator*{\argmin}{argmin} %

\newcommand{\ind}[1]{\mathds{1}\left[{#1}\right]}    %

\def\ddefloop#1{\ifx\ddefloop#1\else\ddef{#1}\expandafter\ddefloop\fi}
\def\ddef#1{\expandafter\def\csname bb#1\endcsname{\ensuremath{\mathbb{#1}}}}
\ddefloop ABCDEFGHIJKLMNOPQRSTUVWXYZ\ddefloop
\def\ddefloop#1{\ifx\ddefloop#1\else\ddef{#1}\expandafter\ddefloop\fi}
\def\ddef#1{\expandafter\def\csname b#1\endcsname{\ensuremath{\mathbf{#1}}}}
\ddefloop ABCDEFGHIJKLMNOPQRSTUVWXYZ\ddefloop
\def\ddef#1{\expandafter\def\csname c#1\endcsname{\ensuremath{\mathcal{#1}}}}
\ddefloop ABCDEFGHIJKLMNOPQRSTUVWXYZ\ddefloop
\def\ddef#1{\expandafter\def\csname h#1\endcsname{\ensuremath{\widehat{#1}}}}
\ddefloop ABCDEFGHIJKLMNOPQRSTUVWXYZabcdefghijklmnopqrsuvwxyz\ddefloop    %
\def\ddef#1{\expandafter\def\csname hc#1\endcsname{\ensuremath{\widehat{\mathcal{#1}}}}}
\ddefloop ABCDEFGHIJKLMNOPQRSTUVWXYZ\ddefloop
\def\ddef#1{\expandafter\def\csname t#1\endcsname{\ensuremath{\widetilde{#1}}}}
\ddefloop ABCDEFGHIJKLMNOPQRSTUVWXYZ\ddefloop
\def\ddef#1{\expandafter\def\csname tc#1\endcsname{\ensuremath{\widetilde{\mathcal{#1}}}}}
\ddefloop ABCDEFGHIJKLMNOPQRSTUVWXYZ\ddefloop

\usepackage{tikz}

\colt{}
\colt{}

\usepackage{prettyref}
\newcommand{\pref}[1]{\prettyref{#1}}

\newcommand{\savehyperref}[2]{\texorpdfstring{\hyperref[#1]{#2}}{#2}}
\newrefformat{eq}{\savehyperref{#1}{\textup{(\ref*{#1})}}}
\newrefformat{eqn}{Equation~\savehyperref{#1}{\ref*{#1}}}
\newrefformat{con}{Conjecture~\savehyperref{#1}{\ref*{#1}}}
\newrefformat{lem}{Lemma~\savehyperref{#1}{\ref*{#1}}}
\newrefformat{def}{Definition~\savehyperref{#1}{\ref*{#1}}}
\newrefformat{line}{line~\savehyperref{#1}{\ref*{#1}}}
\newrefformat{thm}{Theorem~\savehyperref{#1}{\ref*{#1}}}
\newrefformat{corr}{Corollary~\savehyperref{#1}{\ref*{#1}}}
\newrefformat{sec}{Section~\savehyperref{#1}{\ref*{#1}}}
\newrefformat{app}{Appendix~\savehyperref{#1}{\ref*{#1}}}
\newrefformat{apx}{Appendix~\savehyperref{#1}{\ref*{#1}}}
\newrefformat{ass}{Assumption~\savehyperref{#1}{\ref*{#1}}}
\newrefformat{ex}{Example~\savehyperref{#1}{\ref*{#1}}}
\newrefformat{fig}{Figure~\savehyperref{#1}{\ref*{#1}}}
\newrefformat{alg}{Algorithm~\savehyperref{#1}{\ref*{#1}}}
\newrefformat{rem}{Remark~\savehyperref{#1}{\ref*{#1}}}
\newrefformat{conj}{Conjecture~\savehyperref{#1}{\ref*{#1}}} 
\newrefformat{prop}{Proposition~\savehyperref{#1}{\ref*{#1}}}
\newrefformat{proto}{Protocol~\savehyperref{#1}{\ref*{#1}}}
\newrefformat{prob}{Problem~\savehyperref{#1}{\ref*{#1}}}
\newrefformat{claim}{Claim~\savehyperref{#1}{\ref*{#1}}}
\newrefformat{item}{\savehyperref{#1}{\ref*{#1}}}

\usepackage{color-edits}  
\addauthor{as}{red} 

\begin{document}

\newpage 

\setcounter{page}{1}
\maketitle

\begin{abstract}
We consider the {\em learning--unlearning} paradigm defined as follows. First given a dataset, the goal is to learn a good predictor, such as one minimizing a certain loss. Subsequently, given any subset of examples that wish to be {\em unlearnt}, the goal is to learn, without the knowledge of the original training dataset, a good predictor that is identical to the predictor that would have been produced when learning from scratch on the surviving examples.

We propose a new \emph{ticketed} model for learning--unlearning wherein the learning algorithm can send back additional information in the form of a small-sized (encrypted) ``ticket'' to each participating training example, in addition to retaining a small amount of ``central'' information for later.
Subsequently, the examples that wish to be unlearnt present their tickets to the unlearning algorithm, which additionally uses the central information to return a new predictor.
We provide space-efficient ticketed learning--unlearning schemes for a broad family of concept classes, including thresholds, parities, intersection-closed classes, among others.

En route, we introduce the {\em count-to-zero} problem, where during unlearning, the goal is to simply know if there are any examples that survived. We give a ticketed learning--unlearning scheme for this problem that relies on the construction of Sperner families with certain properties, which might be of independent interest.
\end{abstract}

\colt{}

\section{Introduction}

Machine learning models trained on user data have become widespread in  applications. While these models have proved to be greatly valuable, there is an increasing demand for ensuring that
they respect the consent of the users and the privacy of their data. One of the simplest and common challenge is how to update a model when a user wishes to drop out of the training data.
Re-learning the model from scratch without this user's data is a natural approach, but this can be computationally prohibitive.
Designing alternative approaches which aim to ``mimic'' this natural approach, without the computational overhead, has come to be known as the problem of machine ``unlearning'', a topic of growing interest.

Informally, the ideal {\em learning--unlearning (LU)} paradigm can be modeled as follows. An agent gets a sequence of learning and unlearning requests. Each such request is accompanied by a dataset of examples. At any point, the agent must be able to produce a predictor, the distribution of which must be identical to, or at least ``close to'', that of the same agent on a single learning request containing only the {\em surviving} examples, namely, all examples in learning requests that have not been present in any subsequent unlearning request. A naive approach to keep in mind is an agent that explicitly keeps track of all surviving examples at any point, and returns the predictor obtained by simulating a single learning request on the surviving examples. However, the space requirement of such an agent is linear in the number of examples (as it needs to store all the examples to simulate re-training). The goal in this paper is to understand:

\begin{center}
	\sl When are \textbf{space-efficient} learning--unlearning schemes possible?
\end{center}

Unlearning has been an active area of research. \cite{CaoY15} initiated the study through \emph{exact} unlearning.  Their definition requires an algorithm to have identical
outputs on a dataset after deleting a point, as if that point was 
never inserted; their algorithms are restricted to very structured problems. Several later works studied unlearning algorithms for empirical risk minimization ~\citep{GuoGHV20, IzzoSCZ21, NeelRS21, ullah2021machine, thudi2022unrolling, graves2021amnesiac, BourtouleCCJTZLP21}. However, these  works focus exclusively on the {\em time complexity} of unlearning. 
However, many of these provably effective unlearning methods have large space requirements to enable deletion, e.g., they store space-intensive check-pointing data structures \citep{ullah2021machine}, large ensembles of models trained on subsets of data \citep{BourtouleCCJTZLP21}, extensive statistics about the learning process \citep{thudi2022unrolling, NeelRS21}, or at the very least the entire (surviving) training dataset. This additional space overhead can be impractical for various applications.  In contrast, the primary focus of our paper is to understand the {\em space complexity} of unlearning and to develop space-efficient learning--unlearning schemes for general function classes.

The model of learning--unlearning that has been considered in the above mentioned prior works can be deemed as the ``central'' model, where the agent responsible for unlearning has to remember all additional information beyond the learnt predictor that might be required for unlearning later. 

\subsection{Our Contributions}

In \pref{sec:learning_unlearning}, we introduce the notion of a {\em ticketed learning--unlearning (TiLU) scheme}, which consists of a {\em learning} and an {\em unlearning} algorithm.
The learning algorithm given a training dataset, produces a predictor, as well as a ``ticket'' corresponding to each example and some additional ``central'' information.
The unlearning algorithm, given a subset of the training examples accompanied by their corresponding tickets, and the previously generated central information, produces an updated predictor (that realizes the unlearning guarantee).

In \pref{sec:ticketed-lu-schemes}, we show that certain limitations of the central model of learning--unlearning can be overcome by the ticketed model. In particular, we provide TiLU schemes for a broad family of concept classes that includes several commonly studied classes such as one-dimensional thresholds, product of thresholds, and parities. In \pref{sec:sharper}, we provide improved TiLU schemes with even better space complexity bounds for products of thresholds, as well as for point functions, which are notably not covered by the techniques in \pref{sec:ticketed-lu-schemes}.

Underlying our improvements in \pref{sec:sharper}, is a basic primitive of {\em count-to-zero}. The goal in this problem is simply to determine if the unlearning request contains precisely all the examples in the original learning request or not. We give a TiLU scheme for this problem with space complexity that scales as the log of the inverse-Ackermann function (see \definitionref{def:ackermann}) of the number of examples. This relies on a novel construction of {\em size-indexing Sperner families}, which we believe to be of independent interest.  We also prove a lower bound on such families in \pref{sec:lb}, and use it to show that the space complexity of TiLU schemes for any non-trivial concept class must necessarily increase with the number of examples.

\subsection{Other Related Work}

For specific learning models like SVMs, various algorithms for exact unlearning have been considered under the framework of decremental learning \citep{cauwenberghs2001incremental, tveit2003incremental, karasuyama2010multiple, romero2007incremental}. However, these works do not enjoy any guarantees on the space requirements for unlearning. The primary motivation in these works is to use the decremental learning framework to empirically estimate the leave-one-out error in order to provide generalization guarantees for the learned model.  

We also note that beyond exact unlearning, various probabilistic/approximate notions of unlearning, which are in turn inspired by the notions in differential privacy~\citep{DworkMNS06},  have been considered \citep{GinartGVZ19, sekhari2021remember, GuptaJNRSW21, chourasia2022forget}, and there has been an extensive effort to develop  (approximate) unlearning algorithms for various problem settings. These include unlearning in deep neural networks~\citep{du2019lifelong, GolatkarAS20,GolatkarARPS21, nguyen2020variational, graves2021amnesiac}, random forests~\citep{BrophyL21}, large-scale language models~\citep{zanella2020analyzing}, convex loss functions \citep{GuoGHV20, sekhari2021remember, NeelRS21, suriyakumar2022algorithms}, etc. However, we reiterate that all these prior works have huge space requirements, especially in high-dimensional settings. 

In a related setting, prior works have looked at the problem of constructing  {\em history-independent} data structures \citep{hartline2005characterizing, naor2001anti}. The key motivation here is to prevent an adversary from inferring  information about the dataset from the memory representation of the data structure that is not available through its ``legitimate'' interface. However, no direct application of history-independent data structures for unlearning is currently known. 

The main motivation for unlearning is that a trained model could potentially leak user information in various forms such as membership inference attack~\citep{shokri2017membership,carlini2022membership} or even data extraction attack for text~\citep{carlini2021extracting,carlini2022quantifying} or image~\citep{carlini2023extracting} generative models. An unlearning protocol allows users to withdraw their data from the model training set. The implication of such mechanisms were also studied in the literature. For example, \citet{ippolito2022preventing} showed that an inference-time filtering mechanism to prevent the generation of verbatim training text sequences could be easily bypassed by ``style-transfer'' like prompting in large language models. \citet{carlini2022privacy} further show a ``privacy onion'' effect where unlearning one set of examples could have large impact on the privacy score of a non-overlapping set of examples in the training set. Furthermore, there have also been works exploring machine unlearning as a tool to attack an ML system, e.g., \citet{di2022hidden} recently showed that unlearning can be used as a trigger for executing data poisoning attacks (on demand).

Finally, there has been a recent line of work on formulating alternative definitions of approximate (and probabilistic) unlearning to capture the spirit of ``right to be forgotten" and ``data deletion" under different circumstances, e.g.,~\citep{dukler2023safe, krishna2023towards, cohen2022control, eisenhofer2022verifiable, garg2020formalizing}, etc.  However, developing space-efficient algorithms for these definitions is beyond the scope of our paper, and we only consider exact unlearning of ERMs.

\section{Learning--Unlearning Schemes}\label{sec:learning_unlearning}

We focus on supervised learning with the binary loss function $\ell(\hat{y}, y) = \mathds{1}\{\hat{y} \ne y\}$. %
Each example $(x, y)$ belongs to $\cZ = \cX \times \cY$, with the label set $\cY = \bit$. We denote the {\em empirical loss} of a predictor $h : \cX \to \cY$ on a dataset $S = (z_1 = (x_1, y_1), \ldots, z_n = (x_n, y_n)) \in \cZ^n$ as\footnote{For ease of notation, we let $\cL$ denote the sum of losses over the dataset, as opposed to the more conventional average.} $\cL(h; S) := \sum_{i \in [n]} \ell(h(x_i), y_i)$. 
For any concept class $\cH \subseteq (\cX \to \cY)$, we say that a dataset $S$ is \emph{$\cH$-realizable} if there exists $h \in \cH$ such that $\cL(h; S) = 0$. %
Let $\ERM_{\cH}(S) = \argmin_{h \in \cH} \cL(h; S)$ be the set of all minimizers in $\cH$ of the empirical loss over $S$. For any subset $I \subseteq [n]$ of indices, let $S_I$ denote the dataset $((x_i, y_i))_{i \in I}$. 

\paragraph{Central Model of Learning--Unlearning.} 
We define the notion of a {\em learning--unlearning (LU) scheme} below, formalizing the standard setting considered in literature, which we will refer to as the ``central'' model.
\begin{definition}[LU Scheme]\label{def:lu-central}
A {\em learning--unlearning (LU) scheme} for a concept class $\cH$ consists of a pair $(\Learn, \Unlearn)$ of algorithms as follows.
\begin{itemize} 
\item On input $S \in \cZ^n$,  $\Learn(S)$ returns a predictor $h \in \ERM_{\cH}(S)$ and auxiliary information $\aux \in \bit^C$.
\item For $I \subseteq [n]$, %
$\Unlearn(S_I, \aux)$ returns a predictor $h' \in \ERM_{\cH}(S \smallsetminus S_I)$.
\end{itemize}
For all $S$ and $S_I \subseteq S$, the predictor returned by $\Unlearn$ is required to be identical to the predictor returned by $\Learn(S \smallsetminus S_I)$. 
The space complexity of the scheme is $C$, the bit-complexity of the auxiliary information $\aux$.

Unless otherwise stated, we will only require the above to hold for $\cH$-realizable $S$.
\end{definition}
\begin{figure}
\centering \large
\begin{tikzpicture}[
    alg/.style = {rectangle, rounded corners=2pt, draw, minimum height=10mm, minimum width=25mm, line width=0.75pt, outer sep=2pt},
    arrs/.style = {-{Latex[width=3pt,length=3pt]}, line width=0.6pt},
    user/.style = {circle, line width=1pt, draw=Gblue, fill=Gblue!30, minimum size=5mm},
    comm/.style = {fill=white, midway},
    scale=0.63, transform shape
]

\def\ygap{1.35}

\def\xgap{2.4}

\node (S) at (0, 1.2*\ygap) {$S$};
\node[alg, draw=Ggreen, fill=Ggreen!30] (learn) at (0,0) {$\Learn$};
\node (h) at (\xgap, 0) {$h$};
\node at ($(h)+(0.2,-0.6)$) {\scriptsize $\in \ERM_{\cH}(S)$};
\node (aux) at (0, -\ygap) {$\aux$};
\node[alg, draw=Gred, fill=Gred!30] (delete) at (0,-2*\ygap) {$\Unlearn$};
\node (Sp) at (0, -3.2*\ygap) {$S_I$};
\node (hp) at (\xgap, -2*\ygap) {$h'$};
\node at ($(hp)+(0.2,-0.6)$) {\scriptsize $\in \ERM_{\cH}(S \smallsetminus S_I)$};

\path[arrs]
(S) edge (learn)
(learn) edge (h)
(learn) edge (aux)
(aux) edge (delete)
(Sp) edge (delete)
(delete) edge (hp);

\draw[dashed] (4, 1.8*\ygap) -- (4, -3.9*\ygap);

\def\x{5.8}
\node (S) at (\x, 1.2*\ygap) {$S \smallsetminus S_I$};
\node[alg, draw=Ggreen, fill=Ggreen!30] (learn) at (\x,0) {$\Learn$};
\node (aux) at (\x+\xgap, 0) {$\aux''$};
\node (h) at (\x, -\ygap) {$h''$};
\node[right] at ($(h.east)+(-0.15,-0.05)$) {\scriptsize $\in \ERM_{\cH}(S \smallsetminus S_I)$};

\path[arrs]
(S) edge (learn)
(learn) edge (h)
(learn) edge (aux);

\node[rectangle, rounded corners=2pt, draw, text width=2.2cm, align=center, line width=1pt, inner sep=2mm, draw=Gyellow, fill=Gyellow!30] at (\x,-2*\ygap) {
    \textcolor{black!30!Gyellow}{\bf Guarantee:}\\[1mm]
    $h' = h''$
};

\draw[rounded corners=4pt,line width=1.5pt,draw=black!50] (-1.6, 2*\ygap) rectangle (8.8, -4*\ygap);
\node at (3.6, -4.5*\ygap) {\Large (a) LU Scheme (\definitionref{def:lu-central})};

\def\x{12.1}
\def\xgap{2.6}
{\small
\def\ugap{1.4}
\node[user] (u1) at (\x-1.5*\ugap, 1.5*\ygap) {1};
\node[user] (u2) at (\x-0.5*\ugap, 1.5*\ygap) {2};
\node[user] (u3) at (\x+0.5*\ugap, 1.5*\ygap) {3};
\node[user] (u4) at (\x+1.5*\ugap, 1.5*\ygap) {4};

\def\ugap{1.8}
\node[user] (u3p) at (\x-0.5*\ugap, -3.5*\ygap) {3};
\node[user] (u4p) at (\x+0.5*\ugap, -3.5*\ygap) {4};
}
\node[alg, draw=Ggreen, fill=Ggreen!30] (learn) at (\x,0) {$\Learn$};
\node (h) at (\x+\xgap, 0) {$h$};
\node at ($(h)+(0.2,-0.6)$) {\scriptsize $\in \ERM_{\cH}(S_{\crl{1,2,3,4}})$};
\node (aux) at (\x, -\ygap) {$\aux$};
\node[alg, draw=Gred, fill=Gred!30] (delete) at (\x,-2*\ygap) {$\Unlearn$};
\node (hp) at (\x+\xgap, -2*\ygap) {$h'$};
\node at ($(hp)+(0.2,-0.6)$) {\scriptsize $\in \ERM_{\cH}(S_{\crl{1,2}})$};

{\scriptsize
\path[arrs]
(u1) edge[bend right=20] node[comm] {$z_1$} (learn)
($(learn.north)+(-1,0)$) edge[bend right=10] node[comm] {$t_1$} (u1)
(u2) edge[bend right=12] node[comm] {$z_2$} (learn)
(learn) edge[bend right=12] node[comm] {$t_2$} (u2)
(u3) edge[bend right=12] node[comm] {$z_3$} (learn)
(learn) edge[bend right=12] node[comm] {$t_3$} (u3)
(u4) edge[bend right=10] node[comm] {$z_4$} ($(learn.north)+(1,0)$)
(learn) edge[bend right=20] node[comm] {$t_4$} (u4)
(learn) edge (h)
(learn) edge (aux)
(aux) edge (delete)
(u3p) edge node[comm] {$z_3, t_3$} (delete)
(u4p) edge node[comm] {$z_4, t_4$} (delete)
(delete) edge (hp);
}

\draw[dashed] (\x+4.5, 1.8*\ygap) -- (\x+4.5, -3.9*\ygap);

\def\rx{\x+6.4}
\def\ugap{1.8}
{\small
\node[user] (u1) at (\rx-0.5*\ugap, 1.5*\ygap) {1};
\node[user] (u2) at (\rx+0.5*\ugap, 1.5*\ygap) {2};
}
\node[alg, draw=Ggreen, fill=Ggreen!30] (learn) at (\rx,0) {$\Learn$};
\node (aux) at (\rx+\xgap, 0) {$\aux''$};
\node (h) at (\rx, -\ygap) {$h''$};
\node[right] at ($(h.east)+(-0.15,-0.1)$) {\scriptsize $\in \ERM_{\cH}(S_{\crl{1,2}})$};

{\footnotesize
\path[arrs]
(u1) edge[bend right=12] node[comm] {$z_1$} (learn)
(learn) edge[bend right=12] node[comm] {$t_1''$} (u1)
(u2) edge[bend right=12] node[comm] {$z_2$} (learn)
(learn) edge[bend right=12] node[comm] {$t_2''$} (u2)
(learn) edge (h)
(learn) edge (aux);
}

\node[rectangle, rounded corners=2pt, draw, text width=2.2cm, align=center, line width=1pt, inner sep=2mm, draw=Gyellow, fill=Gyellow!30] at (\rx,-2*\ygap) {
    \textcolor{black!30!Gyellow}{\bf Guarantee:}\\[1mm]
    $h' = h''$
};

\draw[rounded corners=4pt,line width=1.5pt,draw=black!50] (\x-2.8, 2*\ygap) rectangle (\x+9.6, -4*\ygap);
\node at (\x+3.4, -4.5*\ygap) {\Large (b) TiLU Scheme (\definitionref{def:lu-ticket})};
\end{tikzpicture}
\caption{Illustration of the guarantees of LU schemes in central and ticketed models.}
\label{fig:lu-scheme}
\end{figure}
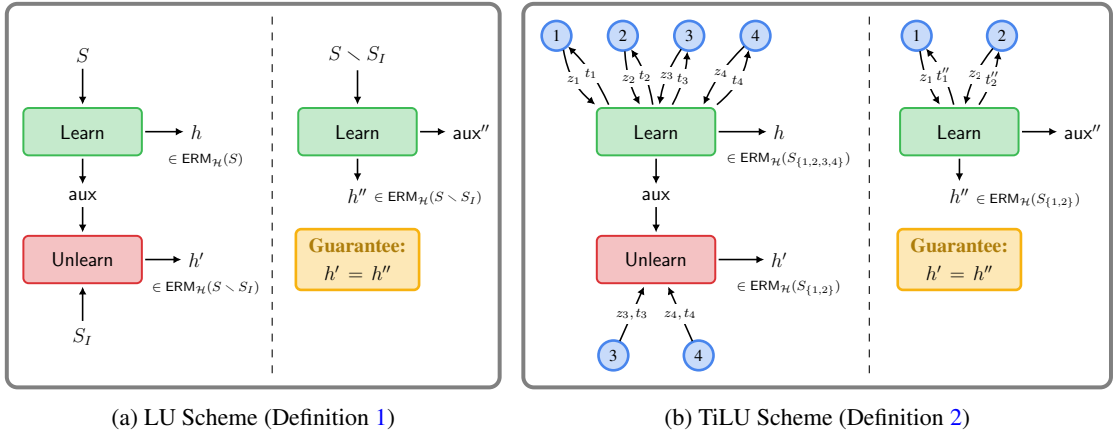

Informally speaking, in a LU scheme, the central agent on a learning request containing a dataset, returns a predictor and also retains some auxiliary information for itself. On an unlearning request (containing a subset of the previous dataset), the agent, using its auxiliary information, returns a new predictor. This new predictor is required to be identical to the predictor that the agent would have returned on a learning request with only the surviving examples. See \pref{fig:lu-scheme}(a).

In this work, our focus is on {\em space complexity}, namely, we consider a LU scheme as {\em space-efficient} if its  space complexity is $\poly(\log n, \log|\cZ|, \log|\cH|)$.  Note that the naive scheme, where $\aux$  contains the entire dataset $S$ and $\Unlearn$ simply returns the predictor returned by $\Learn(S \smallsetminus S_I)$, requires $C = n \cdot \log |\cZ|$ bits and hence is not space-efficient. While we do not explicitly discuss time complexity, all the algorithms that we present are also time-efficient, namely, they run in  $\poly(n, \log |\cZ|, \log |\cH|)$ time. Our lower bounds, on the other hand, hold even against computationally inefficient algorithms. Finally, in this work, we only assume a \emph{one-shot} setting, i.e., there is a single learning request with a dataset \(S\), followed by a single unlearning request \(S_I \subseteq S\).

We show in \pref{apx:central-lu-schemes} that there exists a space-efficient LU scheme for the class of {threshold functions}.  Unfortunately, the central model easily runs into barriers and is quite limited: we also show that any LU scheme for the class of   point functions must store $\Omega(|\cX|)$ bits of auxiliary information.
To circumvent such barriers, we introduce a new {\em ticketed} model of learning--unlearning.

\paragraph{Ticketed Model of Learning--Unlearning.} The basic idea of the ticketed model is that the central agent issues ``tickets'' for each example that is stored by the corresponding user contributing the example. These tickets have to then be provided as part of the unlearning request (see \pref{fig:lu-scheme}(b)).
\begin{definition}[Ticketed LU Scheme]\label{def:lu-ticket}
A {\em ticketed learning--unlearning (TiLU) scheme} for a concept class $\cH$ consists of a pair $(\Learn, \Unlearn)$ of algorithms  such that
\begin{itemize} 
\item On input $S = (z_1, \dots, z_n) \in \cZ^n$, $\Learn(S)$ returns $(h, \aux, (t_i)_{i \in [n]})$, with a predictor $h \in \ERM_{\cH}(S)$, auxiliary information $\aux \in \bit^{C_s}$ and \emph{tickets} $t_1, \ldots, t_n \in \bit^{C_t}$ associated with examples $z_1, \ldots, z_n$ respectively\footnote{The tickets are sent back to the users contributing the corresponding examples, and not stored centrally. }. 
\item On input $I \subseteq [n]$, $\Unlearn(S_I, \aux, (t_i)_{i \in I})$ returns a predictor $h' \in \ERM_{\cH}(S \smallsetminus S_I)$.
\end{itemize}
For all $S$ and $S_I \subseteq S$, the predictor returned by $\Unlearn$ above is required to be identical to the predictor returned by $\Learn(S \smallsetminus S_I)$.
The space complexity of the scheme is $(C_s, C_t)$, where $C_s$ is the bit-complexity of $\aux$ and $C_t$ is the bit-complexity of each ticket $t_i$. 

Unless otherwise stated, we will only require the above to hold for $\cH$-realizable $S$.
\end{definition} 

A TiLU scheme is similar to the central LU scheme, except that during learning, the central agent issues a ``ticket'' for each example in the dataset that will be given to the user contributing that said example; this ticket is required along with the example as part of the unlearning request. As before, we consider a TiLU scheme to be {\em space-efficient} if $C_s, C_t = \poly(\log n, \log |\cZ|, \log |\cH|)$. Note that while the  space requirement for storing all tickets does grow linearly in $n$, the main point is that no single party has to store more than poly-logarithmic amount of information. The challenge in constructing a TiLU scheme is that only the tickets corresponding to the examples in the unlearning request are available at the time of unlearning, so the unlearning step has to be performed with access to a limited amount of information. 
\begin{remark} \label{rem:limitations}\em
For both schemes, our definition is restrictive in the following sense:
\begin{itemize}
\item {\em Exact ERM}: the learning algorithm is required to output a predictor in $\ERM_{\cH}(S)$. %
\item {\em Exact unlearning}: the predictor after unlearning is required to be identical to the predictor learned on just the surviving examples.\footnote{We could consider a seemingly more relaxed variant where the {\em distribution} of the predictor returned by $\Unlearn$ equals the distribution of the predictor returned by $\Learn$ on the surviving examples, but given any such scheme, we could convert it to have a deterministic output by simply choosing a canonical predictor for each distribution.}
\item {\em One-shot unlearning:} there is a single learning request, followed by a single unlearning request.
\end{itemize}
While the above restrictions are seemingly limiting, as our results show, they already capture and highlight many of the key technical challenges.  Developing schemes when these restrictions are suitably relaxed (e.g., when unlearning can be approximate) is an important research direction.
\end{remark}

\section{Mergeable Concept Classes}\label{sec:ticketed-lu-schemes}

We define a {\em mergeability} property of concept classes and provide examples of several commonly studied classes with this property. We then provide space-efficient TiLU schemes for such classes.

\begin{definition}[\boldmath Mergeable Concept Class]\label{def:mergeable}
A concept class $\cH \subseteq (\cX \to \cY)$ is said to be \emph{$C$-bit mergeable} if there exist methods
$\Encode, \Merge, \Decode$ such that
\begin{itemize}
\item $\Encode: \cZ^* \to \bit^C$ is a permutation-invariant encoding of its input into $C$ bits. 
\item $\Decode: \bit^C \to \cH$ such that $\Decode(\Encode(S)) \in \ERM_{\cH}(S)$ for all $\cH$-realizable $S \in \cZ^*$.
\item $\Merge: \bit^C \times \bit^C \to \bit^C$ such that for all $S_1, S_2 \in \cZ^{*}$ such that $S_1 \cup S_2$\footnote{We use $S_1 \cup S_2$ to denote the {\em concatenation} of the two datasets.} is $\cH$-realizable, it holds that $\Encode(S_1 \cup S_2) = \Merge(\Encode(S_1), \Encode(S_2))$.%
\end{itemize}
\end{definition}

\noindent Before we describe TiLU schemes for such classes, we list a few well-studied concept classes that are efficiently mergeable (details deferred to \pref{apx:example-classes}).

\begin{itemize}%
\item \textbf{Thresholds.} Class $\Hth$ consists of all \emph{threshold functions} over $\cX = \crl{1, \dots, |\cX|}$, namely for any $a \in \crl{0, 1, \ldots, |\cX|}$, we have $h_{>a}(x) = \mathds{1}\crl{x > a}$; $\Hth$ is $O(\log |\cX|)$-bit mergeable.

\item \textbf{\boldmath Product of $d$ thresholds.} The class $\Hdth$ over $\cX = [m]^d$, consists of functions indexed by $\ba = (a_1, \dots, a_d) \in \crl{0, 1, \ldots, m}^d$ defined as $h_{>\ba}(x) := \ind{x_1 > a_1 \land \dots \land x_d > a_d}$; $\Hdth$ is $O(d \log |\cX|)$-bit mergeable.

\item \textbf{Parities.} The class $\cH^d_{\oplus}$ consists of all \emph{parity functions}, namely for $\cX = \bbF_2^d$ and $w \in \bbF_2^d$, we have $h_w(x) = \langle w, x \rangle_{\bbF_2}$; $\cH^d_{\oplus}$ is $O(d^2)$-bit mergeable; note that $d = \log |\cX|$.

\item \textbf{Intersection-Closed Classes.}
A class $\cH$ is said to be {\em intersection-closed} if for all $h, h' \in \cH$, the function $\tilde{h}$, given by $\tilde{h}(x) = h(x) \land h'(x)$, is also in $\cH$. Such a class $\cH$ is $d\log |\cX|$-bit mergeable, where $d$ is the VC-dimension of $\cH$. In particular, this includes $\Hdth$ (more examples in \pref{apx:example-classes}).
\end{itemize}

\noindent We also consider an example of a simple class that is {\em not} efficiently mergeable.
\begin{itemize}%
\item \textbf{Point Functions.} Class $\cH_{\point}$ consists of all \emph{point functions} over $\cX = \crl{1, \dots, |\cX|}$, namely for any $a \in \crl{1, \ldots, |\cX|}$, we have $h_a(x) = \mathds{1}\crl{x = a}$; $\cH_{\point}$ is not $o(|\cX|)$-bit mergeable.
\end{itemize}

\subsection{Ticketed LU Schemes for Mergeable Concept Classes}  \label{sec:merkle}

In this section, we provide a TiLU scheme for mergeable concept classes.  The basic idea is to use a {\em Merkle tree}-like data structure to construct tickets for each example.

\begin{theorem}\label{thm:merkle}
For any $C$-bit mergeable concept class $\cH$, there exists a TiLU scheme with space complexity $(C_s = \log |\cH|, C_t = O(C \log n))$.
\end{theorem}

\begin{proof}
For simplicity, we assume that $n$ is a power of $2$; the argument generalizes to all $n$ easily.
Consider a full binary tree of depth $d = \log_2 n$ with leaf $i$ corresponding to example $(x_i, y_i)$.
For each internal node $v$, let $S_v \defeq \crl{(x_i,y_i) \mid v~\text{is an ancestor of leaf}~i}$ be the dataset consisting of examples corresponding to the leaf nodes in the subtree under $v$.
For any leaf $i$, let $v_1, \dots, v_{d - 1}$ be the nodes on the path from the root to leaf $i$, and for each $j \in \crl{2,\ldots, d}$ let $\tilde{v}_j$ be the child of $v_{j-1}$ and sibling of $v_j$. \pref{fig:merkle} shows an example.
Let the ticket corresponding to example $i$ be given as
\[
t_i \defeq \prn{i, \Encode(S_{\tilde{v}_2}), \ldots, \Encode(S_{\tilde{v}_{d}})}.
\]
It is immediate to see that the number of bits in $t_i$ is $d + C\cdot (d-1)$. Define $\Learn(S)$ to return $h = \Decode(\Encode(S))$, $\aux=h$, and tickets $t_i$ as specified above.

We define $\Unlearn$ as follows. If $I$ is empty, then simply return $h$. Given a non-empty $I\subseteq [n]$, let $R$ be the set of all nodes $v$ such that no leaf in the sub-tree under $v$ belongs to $I$, but the same is not true of the parent of $v$. It is easy to see that $S \smallsetminus S_I$ is precisely given as $\bigcup_{v \in R} S_v$, and moreover, $S_{v}$ and $S_{v'}$ are disjoint for distinct $v, v' \in R$. 
For all $v \in R$, we can recover $\Encode(S_v)$ from ticket $t_i$ for any leaf $i \in I$ in the subtree under the sibling of $v$. Thus, by repeated applications  of $\Merge$, we can recover $\Encode(\bigcup_{v \in R} S_v) = \Encode(S \smallsetminus S_I)$. Finally, we can return $h' = \Decode(\Encode(S \smallsetminus S_I))$ as the predictor after unlearning, which is identical to the predictor returned by $\Learn(S \smallsetminus S_I)$.
\end{proof}

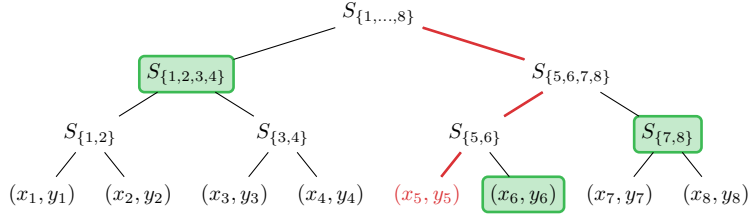
\begin{figure}
\centering \small
\begin{tikzpicture}[
  grow=down,
  level 1/.style={sibling distance=64mm},
  level 2/.style={sibling distance=32mm},
  level 3/.style={sibling distance=16mm},
  level 4/.style={sibling distance=8mm},
  level distance=10mm,
  highlight/.style={rectangle,rounded corners=2pt,draw=Ggreen,fill=Ggreen!30,line width=1pt},
  every node/.style = {outer sep=0pt},
  scale=0.8, transform shape
]
\node (1) {$S_{\crl{1,\ldots,8}}$}
  child {
    node[highlight] {$S_{\crl{1,2,3,4}}$}
    child {
      node {$S_{\crl{1,2}}$}
      child {
        node {$(x_1, y_1)$}
      }
      child {
        node {$(x_2, y_2)$}
      }
    }
    child {
      node {$S_{\crl{3,4}}$}
      child {
        node {$(x_3, y_3)$}
      }
      child {
        node {$(x_4, y_4)$}
      }
    }
  }
  child {
    node (3) {$S_{\crl{5,6,7,8}}$}
    child {
      node (6) {$S_{\crl{5,6}}$}
      child {
        node[Gred] (12) {$(x_5, y_5)$}
      }
      child {
        node[highlight] {$(x_6, y_6)$}
      }
    }
    child {
      node[highlight] {$S_{\crl{7,8}}$}
      child {
        node {$(x_7, y_7)$}
      }
      child {
        node {$(x_8, y_8)$}
      }
    }
  };
\path[Gred, line width=1pt]
(1) edge (3)
(3) edge (6)
(6) edge (12);
\end{tikzpicture}
\caption{Illustration of TiLU scheme underlying the proof of~\theoremref{thm:merkle}. The ticket $t_5$ for the example $(x_5, y_5)$ are the index $i$ and the outputs of $\Encode$ applied on $S_{\crl{1,2,3,4}}$, $S_{\crl{7,8}}$, and $S_{\crl{6}}$.} 
\label{fig:merkle}
\end{figure}

\subsection{Improvements with Compressibility}\label{sec:stable-compress}
We augment the notion of mergeable concept classes, with an additional notion of existence of compressions (similar to the notion introduced by \cite{littlestone1986relating}), and provide improved TiLU schemes for such classes. The advantage of this scheme over that in \theoremref{thm:merkle} is that the space complexity of the ticket depends on \(n\) through only an \emph{additive} $\log n$ factor.

\begin{definition}[\boldmath Mergeable Concept Class with Compressions]\label{def:compression}
A concept class $\cH \subseteq (\cX \to \cY)$ is said to be \emph{$C$-bit mergeable with $K$-compressions} if in addition to $\Encode, \Decode, \Merge$ as in~\definitionref{def:mergeable}, there exists a relation $\Compress \subseteq \cZ^* \times \cZ^{\le K}$, where we say that $T \in \cZ^{\le K}$ is a {\em compression} of $S$ if $(S,T) \in \Compress$. For any $\cH$-realizable $S \in \cZ^*$, the following properties must hold for any valid compression $T$:
(i) $T \subseteq S$,
(ii) $\Encode(S) = \Encode(T)$,
and 
(iii) for all $z \in S \smallsetminus T$, $T$ is a compression for $S \smallsetminus \crl{z}$.
\end{definition}

\noindent The above notion is related, but incomparable to classes with {\em stable} compression schemes~\citep{bousquet2020proper}. A stable compression scheme requires that there exist a unique or canonical compression set for any realizable dataset, whereas in the above we allow the existence of many compressions and do not require that there be a canonical one. On the other hand, the above definition requires the mergeable property, which is not required for stable compression schemes. 

It is easy to see that all the examples of mergeable concept classes we considered earlier are in fact mergeable with compressions (the details are deferred to \pref{apx:example-classes}).
In particular, 
\begin{itemize}[leftmargin=5mm]
 \item Thresholds $\Hth$ are $O(\log |\cX|)$-mergeable with $2$-compressions. 
 \item Parities $\cH^d_{\oplus}$ are $O(d^2)$-mergeable with $d$-compressions.
 \item Intersection-closed classes with VC-dimension $d$ are $O(d \log |\cX|)$-bit mergeable with\linebreak
 $d$-compressions. 
 \end{itemize} 
The following result concerns with TiLU schemes for mergeable concept classes with compressions.

\begin{theorem}\label{thm:span}
For any class $\cH$ that is mergeable with $K$-compressions, there exists a TiLU scheme with space complexity $(C_s = \log |\cH|, C_t = 2K \log |\cZ| + \log n)$. %
\end{theorem}
\begin{proof}
Algorithm $\Learn$ works as follows. Given any dataset $S$, partition $S$ into $T_1, T_2, \ldots$ such that $T_i$ is a compression of $\bigcup_{j \ge i} T_i$ (this can be iteratively done by choosing $T_1$ as any compression of $S$, $T_2$ as any compression of $S \smallsetminus T_1$, and so on). The ticket $t_i$ for the example $(x_i, y_i)$ that lies in $T_j$ is given as $(j, T_j \circ T_{j+1})$. The size of the ticket is $2K \log |\cZ| + \log n$, since $|T_\ell| \le K$ for all $\ell$, and the number of parts is at most $n$. The predictor returned is $h = \Decode(\Encode(S))$ and $\aux = h$.

Algorithm $\Unlearn$ is defined as follows. If $I$ is empty, then simply return $h$. Otherwise, let $\ell$ be the smallest index such that $S_I \cap T_\ell = \emptyset$. We can reconstruct $T_j$ for all $j \le \ell$ from the tickets, since for each $j < \ell$ there exists some example $(x_i, y_i) \in T_j$ that has been presented for unlearning, and the ticket $t_i$ contains $T_j \circ T_{j+1}$. $\Unlearn$ simply returns $\Decode(\Encode(T_1 \cup \cdots \cup T_\ell \smallsetminus S_I))$. $\Learn$ on input $S\smallsetminus S_I$ would have returned $\Decode(\Encode(S \smallsetminus S_I))$. Thus, to be a valid TiLU scheme, it suffices to show that $\Encode(T_1 \cup \cdots \cup T_\ell \smallsetminus S_I) = \Encode(S \smallsetminus S_I)$. This holds because 
\begin{align*}
\lefteqn{ 
\Encode(T_1 \cup \cdots \cup T_\ell \smallsetminus S_I) 
\textstyle~\overset{(*)}{=}~ \Merge(\Encode(T_1 \cup \cdots \cup T_{\ell-1} \smallsetminus S_I), \Encode(T_\ell))} \\
&\textstyle~\overset{(**)}{=}~ \Merge(\Encode(T_1 \cup \cdots \cup T_{\ell-1} \smallsetminus S_I), \Encode(\bigcup_{j \ge \ell} T_j \smallsetminus S_I)) \textstyle~=~ \Encode(S \smallsetminus S_I),
\end{align*}
where, $(*)$ follows from the property of $\Merge$ and $(**)$ uses that $\Encode(T_\ell) = \Encode(\bigcup_{j \ge \ell} T_j \smallsetminus S_I)$, since $T_\ell$ is a compression of $\bigcup_{j \ge \ell} T_j$ and hence a compression of $\bigcup_{j \ge \ell} T_j \smallsetminus S_I$.
\end{proof}

\section{Sharper Bounds for Specific Classes}\label{sec:sharper}

While the previous general reductions are applicable to a large number of concept classes, the resulting space complexity bounds can be improved for certain classes by designing more specific algorithms. In this section, we present TiLU schemes for point functions and (product of $d$) thresholds, as stated formally below. In the case of product of $d$ thresholds, this significantly improves the dependency on $n$ from $\log n$ to $\log \alpha^{-1}(n)$ (inverse-Ackermann function; see \definitionref{def:ackermann}), and in case of 1D thresholds, we also improve the dependency on domain size.

\begin{theorem}\label{thm:sharper-bounds-main}
	There exist TiLU schemes for
	\begin{enumerate}[label=\((\alph*)\)]
		\item $\cH_{\point}$ with space complexity $(O(\log|\cX|), O(\log \alpha^{-1}(n)))$.\label{item:point-ctz} 
		\item $\Hdth$ with space complexity $(O(\log|\cX|), O(\log|\cX| + d \cdot \log \alpha^{-1}(n)))$.\label{item:prod-threshold-ctz}
		\item $\Hth$ with space complexity $(O(\log|\cX|), O(\log \alpha^{-1}(n)))$.\label{item:threshold-ctz}
	\end{enumerate}
\end{theorem}

Note that \Cref{item:threshold-ctz} is an improvement over \Cref{item:prod-threshold-ctz} for $d = 1$, as the ticket size in the former does not depend on $|\cX|$ at all. Moreover, \Cref{item:point-ctz} provides a TiLU scheme for the class of point functions for which the techniques from \theoremref{thm:merkle,thm:span} were not applicable (see \propositionref{prop:pt-func-mergeable-lb}). In Table~\ref{tab:corollaries}, we summarize the implied bounds of \theoremref{thm:merkle,thm:span,thm:sharper-bounds-main} for some concrete classes.

\begin{table}[t] 
	\renewcommand{\arraystretch}{1.1}
	\centering
	\begin{tabular}{|r|c|c|c|}
		\hline
		{\bf\boldmath Class $\cH$} & {\bf\boldmath $C_t$ (from~\theoremref{thm:merkle})} & {\bf\boldmath $C_t$ (from~\theoremref{thm:span})} & {\bf\boldmath $C_t$ (from~\theoremref{thm:sharper-bounds-main})}\\
		\hline
		1D Thresholds & $O(\log |\cX| \cdot \log n)$ & $O(\log |\cX| + \log n)$ & $O(\log \alpha^{-1}(n))$ \\
		Product of $d$ thresholds & $O(d \log |\cX| \cdot \log n)$ & $O(d \log |\cX| + \log n)$ & $O(\log |\cX| + d \log \alpha^{-1}(n))$\\
		Parities ($\cX = \bbF_2^d$) & $O(d^2 \log n)$ & $O(d^2 + \log n)$ & -----\\
  		Point Functions & ----- & ----- & $O(\log \alpha^{-1}(n))$ \\
		\hline
	\end{tabular}
	\caption{The size $C_t$ of tickets, derived as corollaries of \theoremref{thm:merkle} for various concept classes. The size $C_s$ of $\aux$ is $O(\log |\cH|)$ in each case.} 
	\label{tab:corollaries}
\end{table}

At the heart of our improvements is a problem called \emph{count-to-zero}, which we introduce and give a TiLU scheme for in \pref{sec:ctz}. In the subsequent sections, we then use it to give TiLU schemes for each of the aforementioned class. We only describe high-level overviews and the full proof of \theoremref{thm:sharper-bounds-main} is deferred to \pref{app:sharper-proofs}.

\subsection{Count-to-Zero} 
\label{sec:ctz}

\def\rddots#1{\cdot^{\cdot^{\cdot^{#1}}}}

We first describe the \emph{Count-to-Zero} (CtZ) problem. Here there are $m$ examples, where $m$ is unknown a priori. After receiving an unlearning request, we would like to tell whether there is still any example left (that is not unlearned). For CtZ, we follow the terminologies of~\definitionref{def:lu-ticket} except that $\Unlearn$ returns either $\perp$ (when $S_I = S$, and thus no example is left) or $\top$ (when $S_I \subsetneq S$, and some example is remaining), instead of the hypotheses $h, h'$.\footnote{We note that the outputs \(\perp\) or \(\top\) can be stored using only one bit.} 

We give a TiLU scheme for CtZ with ticket size $O(\log \alpha^{-1}(m))$ bits (and one-bit auxiliary information), where $\alpha^{-1}$ is the inverse-Ackermann function defined below. Note that a naive algorithm---writing down $m$ in each ticket---requires ticket size $O(\log m)$ bits.

\begin{definition}\label{def:ackermann}
Consider a sequence of functions $A_1, A_2, \ldots$ defined as, $A_r(1) = 2$ for all $r \ge 1$,
and for all $t \ge 2$,
\begin{itemize}
\item $A_1(t) = 2t$ (interpretable as $t$-fold repeated addition of $2$), and 
\item $A_{r+1}(t) = A_r(A_{r+1}(t-1))$, i.e., $A_{r+1}(t) = A_r^{t-1}(2) = \underbrace{A_r(A_r(\cdots(A_r(2))))}_{t-1 \text{ times}}$.%
\end{itemize}
The \emph{Ackermann function}  $\alpha(t)$ is defined as $A_t(t)$. The \emph{inverse-Ackermann function}  $\alpha^{-1}(n)$ is defined as the smallest $t$ such that $n \le \alpha(t)$.
\end{definition}

\noindent For example, $A_2(t) = 2^t$, and $A_3(t) = 2\upuparrows{t}$, where $a\upuparrows{k}$ denotes the $k$th tetration of $a$ (i.e., $ a^{a^{\rddots a}}$ of order $k$), etc. The Ackermann function was demonstrated as a function that is recursive, but not primitive recursive. 

\begin{theorem} \label{thm:ctz-final-upper-bound} 
	There is a TiLU scheme for \emph{Count-to-Zero} (CtZ) problem with space complexity $(C_s = 1, C_t = O(\log \alpha^{-1}(m)))$. 
\end{theorem} 

\subsubsection{Small-Alphabet Size-Indexing Sperner families} 
\label{sec:sperner}

Recall that a \emph{Sperner family} is a family $\cQ$ of \emph{multi-}sets $(Q_i)_{i \in \bbN}$ such that none of them is a sub-multiset of another (see the book by~\cite{Engel} for background on Sperner families). We say that a Sperner family $\cQ$ is \emph{size-indexing} if for all $m \in \bbN$, the multiset $Q_m$ has $m$ elements (denoted $|Q_m| = m$).  Furthermore, we say that $\cQ$ is of \emph{alphabet size $n(m)$} if for all $m \in \bbN$, each element of the multiset $Q_m$ is from the ordered alphabet $[n(m)] = \{1, \ldots, n(m)\}$. 
For $\ell, r \in \bbN$, $\ell < r$, we define \emph{$[\ell, r]$-size-indexing} Sperner family similar to above, except with $\cQ = \{Q_{\ell}, \dots, Q_r\}$.%

\begin{lemma} \label{lem:sperner-recursive}
	For all $r, t \ge 1$, there is an  $[A_r(t), A_r(A_r(t))]$-size-indexing Sperner family $\cQ^{r,t}$ with alphabet size $2r$.
\end{lemma} \begin{proof}
We prove this by induction over $r$. Consider the base case of $r = 1$, i.e., we want to construct a $[2t, 4t]$-size-indexing Sperner family $\cQ^{1,t}$ with alphabet \(\crl{\texttt{1}, \texttt{2}}\). For each $m \in \{2t, \dots, 4t\}$, let $Q^{1,t}_m$ contain $4t - m$ copies of \texttt{1} and $2m - 4t$ copies of \texttt{2}. It is easy to see that $|Q^{1,t}_m| = m$. Furthermore,
\(\crl{Q^{1,t}_{2t}, \dots, Q^{1,t}_{4t}}\) is a Sperner family since none of the multisets is a sub(multi)set of another multiset (i.e., $m \neq m' \implies Q^{1,t}_m \not\subseteq Q^{1,t}_{m'}$).

Next, consider the case of $r > 1$. Let the alphabet be $\crl{\texttt{1}, \ldots, \texttt{2r}}$. If $t = 1$, then note that $A_r(t) = 2$ and $A_r(A_r(t)) = 4$, and thus from above, there exists $\cQ^{r,t}$ with alphabet size $2$. Next, consider the case of $t \ge 2$. For ease of notation, let $T = A_r(t)$. From the inductive hypothesis we have that there exists an $[A_{r-1}(t'), A_{r-1}(A_{r-1}(t'))]$-size-indexing family $Q^{r-1, t'}$ for all $t' \ge 1$. %
Fix $m \in [T, A_r(T)]$ and let $j_m$ be the unique value such that
$A_{r}(j_m) \le m - T + 2 < A_{r}(j_m+1)$. Since $m \le A_{r}(T)$ we have that $1 \le j_m < T$. We construct $Q^{r,t}_m$ as the union of $T-j_m-1$ copies of $\texttt{2r-1}$ and $j_m-1$ copies of $\texttt{2r}$ and $Q_{m-T+2}^{r-1,j_m}$, where the latter uses symbols $\{\texttt{1}, \ldots, \texttt{2r-2}\}$.

Clearly, $|Q^{r,t}_m| = m$. To see that this is a Sperner family, consider any $m < m'$. Note that since $|Q_m| = m < m' = |Q_{m'}|$ we immediately have $Q_{m'} \nsubseteq Q_m$; thus, it suffices to show that $Q_m \nsubseteq Q_{m'}$. Consider two cases: 

\begin{itemize}[leftmargin=5mm,label=$\bullet$]
\item {\boldmath$j_m = j_{m'} = j$}: Since $Q_m \cap [2r-2] = Q^{r-1,j}_{m - T + 1}$ and $Q_{m'} \cap [2r-2] = Q^{r-1,j}_{m'-T+2}$ are two (different) subsets from the same Sperner family $\cQ^{r-1,j}$, and hence $Q_m \nsubseteq Q_{m'}$. 

\item {\boldmath $j_m < j_{m'}$}: $Q_m$ contains $T - j_m - 1$ copies of $\texttt{2r-1}$, whereas $Q_{m'}$ contains $T - j_{m'} - 1$ copies of $\texttt{2r}$. Since $T - j_m - 1 > T - j_{m'} - 1$, we have $Q_m \nsubseteq Q_{m'}$.
\end{itemize}
\end{proof}

\begin{theorem} \label{thm:sperner-size-indexing-main}
There is a size-indexing Sperner family with alphabet size $O(\alpha^{-1}(m)^3)$.
\end{theorem} 
\begin{proof}
We first show that there exists a $[2, A_t(t)]$-size-indexing Sperner family with alphabet size $t^2$. From \lemmaref{lem:sperner-recursive}, we have $[A_t(i), A_t(i+1)]$-size-indexing Sperner families $\cQ^{t,i}$ using alphabet of size $2(t-1)$ (since $A_t(i) = A_{t-1}(A_t(i-1))$ and $A_t(i+1) = A_{t-1}(A_{t-1}(A_t(i-1)))$). Renaming the symbols to be distinct and setting $\cQ = (\cQ^{t,1}, \cQ^{t,2}, \ldots, \cQ^{t,t})$, we get a
$[2, A_t(t)]$-size-indexing family. In particular, this implies a $[A_{t-1}(t-1), A_t(t) - 1]$-size indexing Sperner family $\cQ^t$.

Finally, this implies a size-indexing Sperner family (for all sizes) with alphabet size $n(m) \le O(m^3)$ as $\cQ = (\cQ^t)_{t \ge 0}$, by renaming symbols so that each $\cQ^t$ uses distinct symbols. (We use $\cQ^{0}$ to denote the trivial $[1, 1]$-size-indexing family with alphabet size of $1$.)
\end{proof}

\subsubsection{From Size-Indexing Sperner families to a TiLU Scheme}

Next, we show that a size-indexing Sperner family can be used to construct a TiLU scheme for CtZ. The basic idea is to use elements of $Q_{|S|}$ as the tickets. This is formalized below. 

\begin{proofof}[\theoremref{thm:ctz-final-upper-bound}]
	For each $m \in \N$, let $Q_m$ be the multiset from the size-indexing Sperner family from \theoremref{thm:sperner-size-indexing-main}; recall $|Q_m| = m$.  Finally, for each $m \in \N$ and $i \in [m]$, we let $q_i^m$ denote the $i$-th element in $Q_m$ (where the elements of $Q_m$ are ordered arbitrarily). 
	Recall also that $S$ is our original training set. Our algorithm is as follows: 
	\begin{align*}
		\Learn(S) =
		\begin{cases}
			(\perp, \perp, \emptyset) & \text{ if } S = \emptyset \\
			(\top, \top, (q_i^{|S|})_{i \in [|S|]}) & \text{ if } S \ne \emptyset,
		\end{cases}
	\end{align*}
	where the auxiliary information ($\top$ or $\bot$) can be written in one bit of $\aux$.
	For unlearning, we define
	\begin{align*}
		\Unlearn(S_I, \aux, (t_i)_{i \in I}) =
		\begin{cases}
			\aux &\text{ if } S_I = \emptyset \\
			\perp &\text{ if } S_I \ne \emptyset \text{ and } \{t_i\}_{i \in I} = Q_{|S_I|} \\
			\top &\text{ if } S_I \ne \emptyset \text{ and } \{t_i\}_{i \in I} \ne Q_{|S_I|},
		\end{cases}
	\end{align*} 
	where the expression $\{t_i\}_{i \in I}$ is interpreted as a multiset.

	The correctness is obvious in the cases $S_I = \emptyset$ and $S_I = S$. We next show correctness when $\emptyset \subsetneq S_I \subsetneq S$. Since $\cQ$ is a Sperner family, we have $Q_{|S_I|} \nsubseteq Q_{|S|}$. However, $\{t_i\}_{i \in I} \subseteq Q_{|S_I|}$ and thus $\Unlearn$ will return $\top$---which is the correct answer---in this case. 
	
	Finally, the space complexity claim follows from the fact that $q_i^{|S|} \subseteq [O(\alpha^{-1}(|S|)^3)]$ since the Sperner family has alphabet size $O(\alpha^{-1}(m)^3)$.  Each ticket only stores a single character from an alphabet of size $O(\alpha^{-1}(m)^3)$ and hence the ticket size is $O(\log \alpha^{-1}(m))$.
\end{proofof}

\subsection{Point Functions} 
We can use CtZ scheme to get a TiLU scheme for $\Hpt$, the class of point function. First, we need to define the predictor $h$ output by $\Learn(S)$. If $(a, 1)$ appears in $S$ for some $a \in \cX$, then we must output $h_a$. Otherwise, we can output $h_b$ for any $b \in \cX$ such that $(b, 0)$ is not in $S$; for tie-breaking, we choose the smallest such $b$. The first case is easy to check: we can just use CtZ to determine whether there is still any $(a, 1)$ left after unlearning. This only requires $O(1)$-size auxiliary information and $O(\log \alpha^{-1}(n))$-size tickets as desired. Now, suppose we are in the second case, and that we start off with $b^*$ being the smallest such that $(b^*, 0)$ is not in $S$. The idea here is to use the CtZ scheme to check whether, after unlearning, any $(b, 0)$ remains in the set for each $b < b^*$. Note that each example $(x, y)$ is only a part of one such subscheme (for $b = x$) and therefore the ticket size remains $O(\log \alpha^{-1}(n))$. However, there seems to be an issue: since we run this subscheme for all $b < b^*$, we may require $O(b^*) \leq O(|\cX|)$ bits to keep the auxiliary information.  Fortunately, this turns out to be unnecessary: Observe that our CtZ scheme never uses the auxiliary information except for the case where the surviving set is empty. On the other hand, if the unlearning request does not contain $(b, 0)$ for $b < b^*$, we know that $S$ must contain $(b, 0)$ for each $b < b^*$ (due to minimality $b^*$). Thus, we do not require storing any additional auxiliary information. This gives the final $C_s = O(\log |\cX|)$ and $C_t = O(\log \alpha^{-1}(n))$ bound. A formal proof is given in \pref{app:sharper-proofs}. 

\subsection{\boldmath Product of \texorpdfstring{$d$}{d} Thresholds}

In the case of product of $d$ thresholds $\Hdth$, we start by deriving a primitive for computing the minimum value of the dataset. 
In the \emph{MINimum VALue} (\minval) problem, we are given a set $S \subseteq \cX$ where $\cX \subseteq \bbR$ is the domain and the goal is to output $\min(S)$ if $S \ne \emptyset$ and $\perp$ if $S = \emptyset$. %

We first give a TiLU scheme for \minval with space complexity $(O(\log|\cX|), O(\log|\cX| + \log \alpha^{-1}(n)))$. To understand the high-level idea, let us consider the case where the input contains only distinct values from $\cX$. In this case, the algorithm is simple: let $\aux=\min(S)$ and the ticket to each $x$ be its successor (i.e., the smallest element in $S$ greater than $x$). When the minimum is unlearned, we update the minimum to its successor. The actual scheme for \minval is more subtle, since we need to handle repeated elements. Nonetheless, the general idea is that we additionally use the CtZ scheme to help determine whether all copies of the minimum have been unlearned. %

To derive a TiLU scheme for $\Hdth$, we use the following learning algorithm for $\Hdth$: Output $h_{>\ba}$ where, for \(j \in [d]\), $a_j$ is the minimum value among all $j$th coordinate of the 1-labeled samples, minus one. To compute this minimum value, we simply use the TiLU scheme for $\minval$. 

\subsection{Further Improvements for 1D Thresholds} 

Our improvement for $\Hth$ is based on an insight from the following (central) LU scheme. First, use binary search to find the threshold. Namely, we start with $a_0 = \lfloor |\cX| / 2\rfloor$. If $h_{>a_0}$ agrees with all the examples, then output $h_{a_0}$. Otherwise, we determine whether $a_0$ is too large or too small, and then adjust the search range and recurse. Suppose that on the original dataset the search path is $a_0, \dots, a_i$. Observe that, after unlearning, the output predictor must be one of $h_{>a_0}, \dots, h_{>a_i}$. Thus, we record (in $\aux$) the empirical loss of each $h_{>a_j}$. We can then update this after seeing the unlearning set. This allows us to determine where in the binary search we stop, and thus the predictor to output. Since binary search examines $O(\log |\cX|)$ hypotheses and recording each empirical loss  uses $O(\log n)$ bits, the space complexity of this LU scheme, without tickets, is $O((\log |\cX|)\cdot (\log n))$.

Now, to reduce the dependency on $n$, we can apply our CtZ scheme. The most natural attempt would be to apply the CtZ scheme on the sets of examples that each $h_{>a_j}$ errs on. However, since each example can be wrong on all of $h_{>a_0}, \dots, h_{>a_{i-1}}$, an example can receive as many as $i - 1 = O(\log |\cX|)$ different tickets from the CtZ scheme. This results in $C_t = O((\log |\cX|)\cdot (\log \alpha^{-1}(n)))$. 

To reduce the ticket size further, observe that whether the empirical error of $h_{>a_j}$ is zero is actually just whether there is an example left between $a_i$ and $a_j$. With this in mind, we partition the domain $\cX$ at points $a_1, \dots, a_j$ and then use the CtZ scheme for each partition. Since we can determine whether each partition is empty, we can determine whether each $h_{>a_j}$ has zero empirical error after unlearning. Furthermore, since each training example belongs to just a single partition, it receives a single ticket of size $O(\log \alpha^{-1}(n))$. This concludes our high-level overview. 

\section{Lower Bounds}\label{sec:lb}
Given the mild dependence on $n$ in the above schemes' space complexity ($O(\log \alpha^{-1}(n))$ bits), it is natural to ask whether this can be removed altogether.  The main result of this section is that it cannot be removed. (In other words, at least one of $C_s$ or $C_t$ must grow with $n$.) Recall that a concept class is \emph{non-trivial} if it contains at least two concepts.

\begin{theorem} \label{thm:lb-learning-main}
	For any non-trivial $\cH$, there does not exist any TiLU scheme for $\cH$ with space complexity $(C_s = O(1), C_t = O(1))$.
\end{theorem}

Once again, our key ingredient is a lower bound on size-indexing Sperner family, which we prove in \pref{sec:lb-sperner}. We then turn this into a space complexity lower bound for CtZ (\pref{sec:lb-ctz}), before finally reducing from this to space complexity of learning any concept class (\pref{sec:lb-ctz-red}).

\subsection{Lower Bounds for Sperner Family}
\label{sec:lb-sperner}

We first derive lower bounds for Sperner families that are more general than size-indexing, as defined below. 

Let $\ba = (a_i)_{i \in \N}$ be any infinite (not necessarily strictly) increasing sequence of integers. We say that a Sperner family $\cQ = (Q_i)_{i \in \N}$ is \emph{$\ba$-size-indexing} if $|Q_i| = a_i$ for all $i \in \N$. (Note that, for the sequence $a_i = i$, this coincides with the size-indexing family from \pref{sec:sperner}.)  The main result of this subsection is 
that any $\ba$-size-indexing family must have alphabet size that grows to infinity:

\begin{theorem} \label{thm:sperner-lb}
For any constant $\sigma \in \N$ and any increasing sequence $\ba$,
there is no $\ba$-size-indexing Sperner family with alphabet size $\sigma$.
\end{theorem}

To prove \theoremref{thm:sperner-lb}, we need the following well-known result
, which is a simple consequence of the so-called ``infinite Ramsey theorem''. Recall that a subset of elements of a partially ordered set (poset) is said to be a \emph{chain} if each element in the subset is comparable to all other elements in the subset. A subset is an \emph{antichain} if no two pair of elements in the subset are comparable.

\begin{lemma}
	[e.g.,~\cite{tachtsis}]
	\label{lem:poset-chain-antichain}
	Any infinite poset either contains an infinite chain or an infinite anti-chain. 
\end{lemma}
\noindent {\bf Proof of \theoremref{thm:sperner-lb}.} 
For convenience, we view each multiset $Q_i$ as a vector in \(\bbN^\sigma\), and let $Q_{i,j}$ denote the number of times $j$ appears in $Q_i$.
We will prove the main statement by induction on $\sigma$. 

\paragraph{Base Case.} The case $\sigma = 1$ is trivial, as for any \(i <  j\) we have \(Q_i \subseteq Q_j\). We next focus on $\sigma = 2$. Suppose for the sake of contradiction that there is an $\ba$-size-indexing family $\cQ$ with alphabet size $2$. Since $|Q_1| = a_1$, we have $Q_{1,1}, Q_{1,2} \leq a_1$. Since $Q_1$ is not a subset of any $Q_i$ for $i > 1$, it must be the case that either $Q_{i, 1} < a_1$ or $Q_{i, 2} < a_1$. Thus, by the pigeonhole principle, there must exist some indices $i_1 < i_2$ (with $i_1, i_2 \leq 2a_1$) 
such that either $Q_{i_1, 1} = Q_{i_2, 1}$ or $Q_{i_1, 2} = Q_{i_2, 2}$. However, this is a contradiction since it implies that $Q_{i_1} \subseteq Q_{i_2}$.

\paragraph{Inductive Step.} Suppose that the statement holds for all $\sigma < m$ for some positive integer $m \geq 3$. Furthermore, suppose for the sake of contradiction that there exists an $\ba$-size-indexing family $\cQ$ with alphabet size $m$. For every $i \in \N$, let $b_i = (Q_{i,1}, Q_{i,2} + \cdots + Q_{i,m})$, and define the natural (partial) order $b_i \leq b_{i'}$ iff $(b_i)_1 \leq (b_{i'})_1$ and $(b_i)_2 \leq (b_{i'})_2$. Applying~\lemmaref{lem:poset-chain-antichain}, this poset must contain an infinite chain or an infinite anti-chain. We consider these two cases separately:
\begin{itemize}
	\item \textbf{Case I}: It contains an infinite anti-chain $(b_{i_j})_{j \in \N}$. Consider the family $\cS := (S_j)_{j \in \N}$ where $S_j$ has elements $1,2$ defined by letting $S_{j,1} = Q_{i_j, 1}$ and $S_{j,2} = Q_{i_j, 2} + \cdots + Q_{i_j, m}$. Notice that $\cS$ is an $(a_{i_j})_{j \in \N}$-size indexing family. Furthermore, since $b_{i_j}$ is an anti-chain, this is a Sperner family. This contradicts our inductive hypothesis for $\sigma = 2$.
	\item \textbf{Case II}: It contains an infinite chain $(b_{i_j})_{j \in \N}$, i.e., where $b_{i_1} < b_{i_2} < \cdots$. Consider $\cS := (S_j)_{j \in \N}$ where $S_j$ results from throwing away all ones from $Q_{i_j}$. $\cS$ is an $((b_{i_j})_2)_{j \in \N}$-size indexing family, and $((b_{i_j})_2)_{j \in \N}$ is a increasing sequence (because $(b_{i_j})_{j \in \N}$ forms a chain). Finally, since $Q_{i_1,1} \leq Q_{i_2,1} \leq \cdots$, $\cS$ must form a Sperner family\footnote{Otherwise, if $S_{j} \subseteq S_{\ell}$ for some $j < \ell$, we must have $Q_{i_j} \subseteq Q_{i_{\ell}}$ since $Q_{i_j,1} \leq Q_{i_{\ell},1}$.}. Since this collection only uses alphabet $2, \dots, m$, this contradicts our inductive hypothesis for $\sigma = m -1$. 
\end{itemize}
In both cases, we arrive at a contradiction, concluding the proof. \hfill $\blacksquare$ 

\subsection{From Sperner Family to CtZ}
\label{sec:lb-ctz}

Next, we observe that the lower bound on the alphabet size for Sperner Family translates to a lower bound for CtZ. This is the ``reverse reduction'' of \theoremref{thm:ctz-final-upper-bound}, but this direction requires more care as we have to handle the auxiliary information $\aux$ as well.

\begin{lemma} \label{lem:sperner-to-ctz-lb}
	There is no TiLU scheme for CtZ with space complexity  $(C_s = O(1), C_t = O(1))$.
\end{lemma}

\subsection{From Learning any Classes to CtZ}
\label{sec:lb-ctz-red}

Finally, we observe that any TiLU scheme for the learning/unlearning problem can be converted into a TiLU scheme for CtZ with the same complexity, formalized below in~\theoremref{thm:count-to-zero-lb-red}.
This, together with~\lemmaref{lem:sperner-to-ctz-lb}, yields \theoremref{thm:lb-learning-main}.

\begin{lemma} \label{thm:count-to-zero-lb-red}
	For any non-trivial $\cH$, if there exists a TiLU scheme for $\cH$ with space complexity $(C_s, C_t)$, then there also exists a TiLU scheme for CtZ with the same space complexity  $(C_s, C_t)$.
\end{lemma}

\section{Future Directions} 
In this paper, we introduce and study the ticketed model of unlearning and obtain several results. Perhaps the most basic unresolved question is to characterize the classes for which space-efficient TiLU schemes exist. For example, can we get TILU schemes for all concept classes with VC-dimension $d$, whose space complexity scales as \(\poly(d, \log(n), \log(\abs{\cX}))\)? Note that we do not know the answer to this question even for some basic classes such as linear separators.

Another tantalizing open question is to make the lower bound for CtZ (\theoremref{thm:count-to-zero-lb-red}) quantitative. Related to this, it would also be interesting to understand whether subpolylogarithmic in $n$ dependency is always possible for all space-efficient TiLU schemes.

We already pointed out several limitations of our model in~\remarkref{rem:limitations}. Extending the TiLU model to overcome these limitations, e.g., allowing approximate ERM, seems like an interesting research direction. Another intriguing direction is to study unlearning in the agnostic setting where the dataset is not assumed to be $\cH$-realizable.  It is unclear how to obtain generic space-efficient LU schemes, such as ones in \theoremref{thm:merkle} and \theoremref{thm:span}, in the agnostic setting; as a preliminary result, we give a TiLU scheme for agnostic 1D thresholds in ~\pref{app:1d-threshold-ticket-ub}. Finally, it is also interesting to study the ticketed model for simpler tasks such as realizability, testing, etc; \pref{app:realizability_testing} contains initial results for agnostic 1D thresholds.

\subsubsection*{Acknowledgments}  
Part of the work was done when AS was a student researcher at Google Research, Mountain View during summer 2022. AS thanks Surbhi Goel, Sasha Rakhlin, and Karthik Sridharan for useful discussions, and acknowledges the support of NSF through awards DMS-2031883 and DMS-1953181, as well as DOE through award DE-SC0022199. 

\clearpage 

\bibliography{refs,alt_references}

\newpage
\appendix

\section{Examples of Mergeable Concept Classes with Compression}\label{apx:example-classes}

We recall the function classes defined in \pref{sec:ticketed-lu-schemes} and show that they are mergeable concept classes with compressions (as in~\definitionref{def:compression}).

\paragraph{1D Thresholds.}%
The class $\Hth$ consists of all threshold functions over $\cX = \crl{1, \dots, |\cX|}$, namely for any $a \in \crl{0, 1, \ldots, |\cX|}$, we have $h_{>a}(x) = \mathds{1}\crl{x > a}$. The class $\Hth$ is $O(\log |\cX|)$-bit mergeable with 1-compressions, as witnessed by the following methods:
\begin{itemize}[leftmargin=6mm]
\item For any $\Hth$-realizable dataset $S$, we set $\Encode(S) = x_-$, where $x_-$ is the largest $x_i$ with corresponding $y_i = 0$ or $0$ if no such $x_i$ exists.
\item $\Decode(x_-) = h_{>x_-}$.
\item $\Merge(x^1_-, x^2_-) = \max(x^1_-, x^2_-)$.
\item $\Compress$: If $x_- \ne 0$ for $S$, then $(S, T) \in \Compress$ iff $T = \{(x_-, 0)\}$. Otherwise, if $x_- = 0$ for $S$, then $(S, T) \in \Compress$ iff $T = \emptyset$.
\end{itemize}

\paragraph{Parities.}
The class $\cH^d_{\oplus}$ consists of all parity functions, namely, for $\cX = \bbF_2^d$ and $w \in \bbF_2^d$, we have $h_w(x) = \langle w, x \rangle$. The class $\cH^d_{\oplus}$ is $O(d^2)$-bit mergeable with $d$-compressions, where $d = \log |\cX|$.
\begin{itemize}[leftmargin=6mm]
\item For any $\cH^d_{\oplus}$-realizable dataset $S$, we set $\Encode(S) = W$, where $W$ is a representation of the set of all $w$ such that $\langle w, x_i \rangle = y_i$ for all $(x_i, y_i) \in S$. Note that this set corresponds to an affine subspace in $\bbF_2^d$ and thus, can be represented using $O(d^2)$ bits.
\item $\Decode(W) = h_w$ where $w$ is the lexicographically smallest vector in $W$.
\item $\Merge(W_1, W_2)$ is simply a representation of $W_1 \cap W_2$, which is also an affine subspace.
\item $\Compress$ consists of all $(S, T)$ such that $T$ is a minimal subsequence of $S$ that generates the same $h_w$. By standard linear algebra, it follows that there are at most $d$ such examples.
\end{itemize} 

\paragraph{Intersection-Closed Classes.}
A class $\cH$ is said to be {\em intersection-closed} if for all $h, h' \in \cH$, the ``intersection'' $\tilde{h}$, given by $\tilde{h}(x) = h(x) \land h'(x)$, is also in $\cH$. Any intersection-closed class $\cH$ is $d\log |\cZ|$-bit mergeable  with $d$-compressions, where $d$ is the VC-dimension of $\cH$. The key property we use to show this is \lemmaref{lem:intersection-closed}.  

\begin{lemma}[\cite{auer2007new}]\label{lem:intersection-closed} 
For any intersection-closed class $\cH$ with VC-dimension $d$, for all $\cH$-realizable $S$, 
there exists a unique $h_S \in \cH$ such that for all $h \in \ERM_{\cH}(S)$, it holds for all $x \in \cX$ that $h_S(x) = 1 \implies h(x) = 1$. Moreover, there exists $T \subseteq S$ such that $|T| \le d$ and $y_i = 1$ for all $(x_i, y_i) \in T$ and $h_T = h_S$.
\end{lemma} 

For any intersection-closed class $\cH$ with VC-dimension $d$, we can consider $\Encode$, $\Decode$, and $\Merge$ as follows:
\begin{itemize}
\item For any $\cH$-realizable dataset $S$, we set $\Encode(S) = h_T$, where $h_T$ is as given by~\lemmaref{lem:intersection-closed}.
It is easy to see that $h_T$ can be represented using only $d \log |\cX|$ bits. 
\item $\Decode(\tilde{h}) = \tilde{h}$.
\item $\Merge(\tilde{h}_1, \tilde{h}_2)$ is simply the representation of $\tilde{h}_1 \land \tilde{h}_2$.
\item $\Compress$ consists of all $(S, T)$ such that $T$ is the minimal subsequence of $S$, as guaranteed to exist by \lemmaref{lem:intersection-closed}.
\end{itemize}

\noindent
We now mention some examples of intersection closed classes with VC-dimension $d$. 

\begin{itemize}[leftmargin=6mm]
\item Product of $d$ thresholds $\Hdth$, defined in \pref{sec:ticketed-lu-schemes}.%
\item Product of $d$ thresholds in $n$ dimensions over $\cX = \crl{1,\ldots,m}^n$ given as
\[\textstyle
\cH^{d,n}_{\thresh} := \left\{h_{>\ba}(x) = \ind{\bigwedge_{i \in [n]} x_i > a_i} : \ba \in \crl{0,\ldots,m}^n, \abs{\crl{i : a_i \neq 0}} \leq d\right\}.
\]
\item Intersection of half-spaces from a fixed set of possible orientations $\cV = \crl{v_1, \dots, v_d} \in \bbR^n$: 
\begin{align*}
\textstyle \cH_{\cV} := \crl{h_{B}(x) = \bigcap_{v \in B} \ind{\tri{x, v} \leq 1} \mid B \subseteq \cV}.
\end{align*}
While this involves a domain of infinite size, one could consider a finite (discretized) version of the same.
\item $d$-point functions: $\cH^{d}_{\point} := \crl{h_{A}(x) = \ind{x \in A} ~:~ A \subseteq \cX, \abs{A} \leq d}$.
\end{itemize}

\subsection{Lower Bounds for Mergeable Concept Classes}

Recall that the class $\cH_{\point}$ consists of all point functions over $\cX = \crl{1, \dots, |\cX|}$, namely for any $a \in \cX$, we have $h_a(x) = \mathds{1}\crl{x = a}$. %
In contrast to the concept classes above, we show that $\cH_{\point}$ is \emph{not} efficiently mergeable, as stated below. Note that this lower bound is nearly tight, since recording the entire input dataset requires only $O(|\cX| \log n)$ bits.

\begin{proposition}\label{prop:pt-func-mergeable-lb}
$\cH_{\point}$ is not $o(|\cX|)$-bit mergeable.
\end{proposition}

\noindent We prove \propositionref{prop:pt-func-mergeable-lb} using a lower bound on one-way communication complexity of the $\mathsf{Index}_m$ problem.

\begin{definition}\label{def:index}
The $\mathsf{Index}_m$ problem is given as:
Alice gets an input $X \in \bit^m$, Bob gets an input $i \in [m]$, and the goal is for Bob to output $X_i$ with probability at least $3/4$, where Alice and Bob have access to shared randomness.
\end{definition}

\begin{lemma}[\cite{KremerNR99}]\label{lem:indexing}
Any one-way (Alice $\to$ Bob) protocol for $\mathsf{Index}_m$ has communication $\Omega(m)$.
\end{lemma}

\begin{proofof}[\propositionref{prop:pt-func-mergeable-lb}]
Let $\cX = \{1, \ldots, m+1\}$.
Suppose that $\cH_{\point}$ is $C$-bit mergeable. We show that this implies a $C$-bit protocol for $\mathsf{Index}_{m}$. Consider the following one-way communication protocol for $\mathsf{Index}_m$ using a shared random permutation $\pi: \cX \to \cX$.

\paragraph{Alice:} On input $X \in \bit^m$: 
\begin{itemize}
\item Let $S_1$ be a dataset containing $X_j$ copies of $(\pi(j), 0)$ for all $j \in [m]$.
\item Send $\Encode(S_1)$ to Bob.
\end{itemize}

\paragraph{Bob:} On input $i \in [m]$ and message $E = \Encode(S_1)$ from Alice:
\begin{itemize}
\item Let $S_2$ be the dataset containing one copy of $(\pi(j), 0)$ for all $j \in [m] \smallsetminus \{i\}$.
\item If $\Decode(\Merge(E, \Encode(S_2))) = h_{\pi(i)}$, return $Y_i = 0$. Otherwise, return $Y_i = 1$.\\
\end{itemize} 

\noindent Note that $S_1 \cup S_2$ is always realizable since $h_{\pi(m+1)} \in \ERM_{\Hpt}(S_1 \cup S_2)$. To show that this communication protocol solves $\mathsf{Index}_m$, we show the following:

\paragraph{\boldmath Claim: $\Pr[Y_i = 1 ~\mid~ X_i = 1] = 1$.}\mbox{}\\
When $X_i = 1$, $(\pi(j), 0)$ appears at least once in $S_1 \cup S_2$ for all $j \in [m]$. Therefore, the unique ERM for $S_1 \cup S_2$ is $h_{\pi(m+1)}$ and not $h_{\pi(i)}$.

\paragraph{\boldmath Claim: $\Pr[Y_i = 1 ~\mid~ X_i = 0] = 1/2$.}\mbox{}\\
It suffices to prove the statement for $\pi$ that is fixed on all coordinates except for $\pi(i)$ and $\pi(m+1)$. Let $x_1, x_2$ denote the remaining elements of $\cX$ (to be assigned to $\pi(i), \pi(m + 1)$). Once the rest of $\pi$ is fixed and $X_i = 0$, $\Decode(\Encode(S_1 \cup S_2))$ is one of $h_{x_1}$ or $h_{x_2}$, since $S \smallsetminus \U$ contains $(x,0)$ for all $x \in \cX \smallsetminus \{x_1, x_2\}$ and so $h_{x} \notin \ERM_{\Hpt}(S\smallsetminus \U)$. Without loss of generality, suppose $\Decode(\Encode(S_1 \cup S_2)) = h_{x_1}$. Thus, we have 
$\Pr[Y_i = 1 \mid X_i = 1, \pi_{-\{i, m+1\}}] = \Pr[\pi(i) = x_1 \mid \pi_{-\{i, m+1\}}] = 1/2$.\\ 

\noindent We can solve $\mathsf{Index}_m$ with probability $3/4$ by repeating the above protocol twice, returning $1$ if each step returns $1$. This ensures correctness with probability $3/4$. Thus, we get that $C = \Omega(m) = \Omega(|\cX|)$. 
\end{proofof}

\section{LU Schemes and Lower Bounds in the Central Model}\label{apx:central-lu-schemes}

In this section we discuss LU schemes and lower bounds (in the central model) for simple function classes of 1D thresholds $\Hth$ and point functions $\Hpt$ (see \pref{apx:example-classes} for definitions).

\subsection{LU Scheme for 1D Thresholds}\label{subapx:central-lu-thresholds}

\begin{theorem}\label{thm:1d-thresh-central-ub}
There exists a LU scheme for $\Hth$ with space complexity $C = O\prn{\prn{\log |\cX|}\prn{\log n}}$, that is valid for $\Hth$-realizable datasets.
\end{theorem}
\begin{proof}
We describe the methods $\Learn$ and $\Unlearn$, given a $\Hth$-realizable dataset $S$ and unlearning request $\U$. For simplicity, we will assume that $\cX = \{1, \ldots, 2^d-2\}$, i.e.,~$|\Hth| = |\cX| + 1 = 2^d - 1$ for some integer $d$. 

\paragraph{\boldmath $\Learn$.}
Since $S$ is $\Hth$-realizable, it follows that $\ERM_{\Hth}(S)$ is an ``interval'' $\crl{h_{>p}, \ldots, h_{>q}}$, where $p$ is the largest $x_i$ in $S$ with $y_i = 0$ (or $0$ if none exists) and $q+1$ is the smallest $x_i$ in $S$ with $y_i = 1$ (or $2^d$ if none exists). On input $S$, let $h_{>a_0}, h_{>a_1}, \ldots$ be the sequence of predictors obtained when performing a binary search. Such a sequence can be obtained using the following procedure.

\begin{itemize}[topsep=3mm]
\item $i \gets 0$ and $a_0 \gets 2^{d-1} - 1$
\item {\bf while} $a_i < p$ or $a_i > q$
    \begin{itemize}[label=$\triangleright$]
    \item $a_{i+1} \gets \begin{cases} 
        a_{i} + 2^{d-i-2} & \text{if } a_{i} < p\\
        a_{i} - 2^{d-i-2} & \text{if } a_{i} > q
        \end{cases}$
    \item $i \gets i+1$
    \end{itemize}
\item {\bf return} $h_{>a_0}, h_{>a_1}, \ldots, h_{>a_{i}}$
\end{itemize}

We define $\Learn(S)$ to return the predictor $h_{>a_i}$ (last one in the sequence), and auxiliary information
$\aux = (a_i, \err_0, \err_1, \ldots, \err_i)$, where $\err_j := \cL(h_{>a_j}; S)$. Since $i \leq d$, we have that the bit representation of $\aux$ is at most $O(d \cdot \log n)$.

\paragraph{\boldmath $\Unlearn$.}
For any $\U \subseteq S$, we have that $\ERM_{\Hth}(S \smallsetminus \U) \supseteq \ERM_{\Hth}(S)$. Thus, $\ERM_{\Hth}(S \smallsetminus \U)$ is given by the interval $\crl{h_{>p'}, \ldots, h_{>q'}}$ for $p' \le p \le q \le q'$. The sequence generated by binary search on $S \smallsetminus \U$ would be the same as $h_{>a_0}, h_{>a_1}, \ldots$ except it might stop earlier.

We define $\Unlearn(\U, \aux)$ as follows. First, using $a_i$, it is simple to see that we can recover $a_0, \dots, a_{i - 1}$, as there is a unique binary search path to reach $a_i$. Then, for $j = 0, \ldots, i$, compute $\cL(h_{>a_j}; S \smallsetminus \U) = \err_j - \cL(h_{>a_j}; \U)$, and return the predictor $h_j$, where $j$ is the smallest index with $\cL(h_{>a_j}; S \smallsetminus \U) = 0$.
\end{proof} 

We note that the unlearning algorithm in the proof of~\theoremref{thm:1d-thresh-central-ub} even allows for multi-shot unlearning (which is not the focus of this paper).

\subsection{Lower Bounds on LU Schemes for Point Functions}

While we showed that the class $\Hth$ admits a space-efficient LU scheme, the same is not the case for the class $\Hpt$ of point functions: we show that any LU scheme must have space complexity at least linear in the number of examples. We, however, also show two ways to circumvent this barrier. First, if we augment the class $\Hpt$ to $\overline{\cH}_{\point}$ to also include the {\em zero} function, then we can obtain a space-efficient LU scheme. Alternatively, if we only restrict to datasets without any repetitions, then again there exists a space-efficient LU scheme. 

\begin{theorem}\label{thm:point-fn-central}
For the class $\Hpt$ of point functions,
\begin{enumerate}[label=\((\alph*)\)] 
\item Any LU scheme for $\Hpt$, even those valid only for $\Hpt$-realizable datasets, must have space complexity $\Omega(|\cX|)$ when $n \ge \Omega(|\cX|)$. 
\item There exists an LU scheme for $\overline{\cH}_{\point}$ with space complexity $C = O(\log |\cX| + \log n)$, that is valid for $\overline{\cH}_{\point}$-realizable datasets.
\item There exists an LU scheme for $\Hpt$ with space complexity $C = O(\log |\cX|)$, that is valid for $\Hpt$-realizable datasets without repetitions. 
\end{enumerate} 
\end{theorem} 

We note that part-\((a)\) and part-\((b)\), taken together, imply a separation between proper and improper learning in terms of the storage needed for LU schemes. While proper learning requires \(\Omega(n \wedge \abs{\cX})\) bits in the central memory, improper learning can be done with logarithmically many bits. 

\subsubsection{Lower Bounds on Space-Complexity of LU Schemes for \texorpdfstring{$\Hpt$}{Hpt}} 

We prove \theoremref{thm:point-fn-central}(a), using the lower bound on one-way communication complexity of the $\mathsf{Index}_m$ problem~(\lemmaref{lem:indexing}).

\begin{proofof}[\theoremref{thm:point-fn-central}(a).]
Consider an LU scheme for $\Hpt$ with $\Learn$ and $\Unlearn$ methods. Let $\cX = \{1, \ldots, m+1\}$. Consider the following one-way communication protocol for $\mathsf{Index}_m$ using a shared random permutation $\pi: \cX \to \cX$.

\paragraph{Alice:} On input $X \in \bit^m$: 
\begin{itemize}
\item Let $S$ be a dataset containing $X_i + 1$ copies of $(\pi(i), 0)$ for all $i \in [m]$. 
\item Let $(h, \aux) \gets \Learn(S)$
\item Send $\aux$ to Bob.
\end{itemize}

\paragraph{Bob:} On input $i \in [m]$ and message $\aux$ from Alice:
\begin{itemize}
\item Let $h \gets \Unlearn(\U, \aux)$ for $S_I := \{(\pi(i), 0)\}$.
\item If $h = h_{\pi(i)}$, return $Y_i = 0$. Otherwise, return $Y_i = 1$.\\
\end{itemize} 

\noindent To show that this communication protocol solves $\mathsf{Index}_m$, we show the following:

\paragraph{\boldmath Claim: $\Pr[Y_i = 1 ~\mid~ X_i = 1] = 1$.}\mbox{}\\
When $X_i = 1$, $(\pi(i), 0)$ appears twice in $S$. Therefore, even after unlearning $\U$, a copy of $(\pi(i), 0)$ remains in $S \setminus S_{I_n}$. Thus, $h_{\pi(i)}$ cannot be the ERM and therefore, $Y_i  = 1$. 

\paragraph{\boldmath Claim: $\Pr[Y_i = 1 ~\mid~ X_i = 0] = 1/2$.}\mbox{}\\ 
It suffices to prove the statement for $\pi$ that is fixed on all coordinates except for $\pi(i)$ and $\pi(m+1)$. Let $x_1, x_2$ denote the remaining elements of $\cX$ (to be assigned to $\pi(i), \pi(m + 1)$). Once the rest of $\pi$ is fixed and $X_i = 0$, the predictor returned by $\Learn(S \setminus \U)$ is one of $h_{x_1}$ or $h_{x_2}$, since $S \smallsetminus \U$ contains $(x,0)$ for all $x \in \cX \smallsetminus \{x_1, x_2\}$ and so $h_{x} \notin \ERM_{\Hpt}(S\smallsetminus \U)$. Without loss of generality, suppose the predictor is $h_{x_1}$. Thus, we have 
$\Pr[Y_i = 1 \mid X_i = 1, \pi_{-\{i, m+1\}}] = \Pr[\pi(i) = x_1 \mid \pi_{-\{i, m+1\}}] = 1/2$.\\ 

\noindent We can solve $\mathsf{Index}_m$ with probability $3/4$ by repeating the above protocol twice, returning $1$ if each step returns $1$. This ensures correctness with probability $3/4$. Thus, we get that the bit complexity of $\aux$ must be $\Omega(m) = \Omega(|\cX|)$.
\end{proofof}

\subsubsection{LU Scheme for Point Functions Augmented with Zero}

\begin{proofof}[\theoremref{thm:point-fn-central}(b).] 
Recall that $\overline{\cH}_{\point}$ consists of functions $h_a(x) = \indic{x = a}$ for all $a \in \cX$, as well as, $h_0(x) = 0$ for all $x \in X$. The key idea is that on input $S$, the learning algorithm returns a predictor $h_a$ if the example $(a,1)$ appears in the dataset, else returns the predictor $h_0$.

\paragraph{\boldmath $\Learn$.} On the input $S$, let $h = h_a$ if example $(a, 1)$ appears in $S$ for some $a$, otherwise if all the labels are $0$, let $h = h_0$. Note that there can be at most one such $a$ since $S$ is $\Hpt$-realizable.

When $h = h_a$, the example $(a, 1)$ can appear more than once. So we set $\aux = (h, c)$, where $c$ is the number of times $(a,1)$ appears in $S$, or $0$ if $h = h_0$. Thus, the bit representation of $\aux$ is $\log |\cH| + \log n$.

\paragraph{\boldmath $\Unlearn$.} On input $\U \subseteq S$, and $\aux = (h,c)$ if $h = h_0$ return $h_0$. Else, if $h = h_a$ and the number of times $(a,1)$ appears in $\U$ is equal to $c$, then return $h_0$, else return $h_0$. Basically, using $\aux$, we are able to verify if $S \smallsetminus \U$ contains any example with label $1$ or not.
\end{proofof}

\subsubsection{LU Scheme for Datasets Without Repetitions}

\begin{proofof}[\theoremref{thm:point-fn-central}(c).] The key idea is that on input $S$, the learning algorithm returns the predictor $h_a$ where either $(a,1)$ appears in $S$ or $a$ is the smallest value such that $(a,0)$ does not appear in $S$. With that in mind, we define $\Learn$ and $\Unlearn$ as follows.

\paragraph{\boldmath $\Learn$.} On input \(S\), 
\begin{itemize}
	\item If $(a,1)$ appears in $S$ for some $a$, then return the predictor $h_a$ and $\aux = (a, b)$, where $b$ is the smallest value such that $(b,0)$ does not appear in $S$. 

	\item If all labels in $S$ are $0$, then let $a$ be the smallest value such that $(a,0)$ does not appear in $S$, and return the predictor $h_a$ and $\aux = (a, a)$.
\end{itemize}
It is easy to see that the bit complexity of $\aux$ is $2 \log |\cX|$.

\paragraph{\boldmath $\Unlearn$.} On input $\U$ and $\aux = (a, b)$, return $h_a$ if $\U$ is empty. Otherwise, 
\begin{itemize}
    \item if $a \ne b$ and $(a, 1) \notin \U$, then return $h_a$.
    \item else, let $c$ be the smallest value such that $(c,0) \in \U$ return $h_{\min\crl{b,c}}$.
\end{itemize}
\end{proofof}

\section{Separation between Central and Ticketed LU Models: Agnostic Case}\label{apx:separations}

Note that \theoremref{thm:sharper-bounds-main} and \theoremref{thm:point-fn-central} show that the class of point functions do not have space-efficient LU schemes, but do have space-efficient TiLU schemes.
In this section, we provide yet another example for which there is a space-efficient TiLU scheme, but there does not exist any space-efficient LU scheme. In particular, this example holds for the class of 1D thresholds in the {\em agnostic} setting. We show the following: 

\begin{theorem}\label{thm:thresh-separation}
For the class $\Hth$ of 1D thresholds, 
\begin{enumerate}[label=\((\alph*)\)] 
\item Any LU scheme that is valid for all (even unrealizable) datasets must have space complexity $\Omega(\min\crl{n, |\cX|})$.
\item There exists a TiLU scheme valid for all (even unrealizable) datasets with space complexity $(C_s = \log |\cX|, C_t = \log |\cX| (\log |\cX| + \log n))$.
\end{enumerate}
\end{theorem}
\noindent 
In contrast, if \(S\) is \(\Hth\)-realizable, then \pref{thm:1d-thresh-central-ub} shows a space-efficient LU scheme in the central model.

\subsection{Lower Bounds for Central LU Schemes}
\label{app:1d-threshold-lb-central} 

Suppose for contradiction we have an LU scheme for $\Hth$, consisting of $\Learn$ and $\Unlearn$ methods, that is valid even in the agnostic setting. Suppose $|\cX| \ge 4m+4$.
We will construct a two-party one-way communication protocol, with communication complexity being the space complexity of the LU scheme, and where Bob can entirely reconstruct Alice's input $X \in \bit^m$ (with probability $1$). 

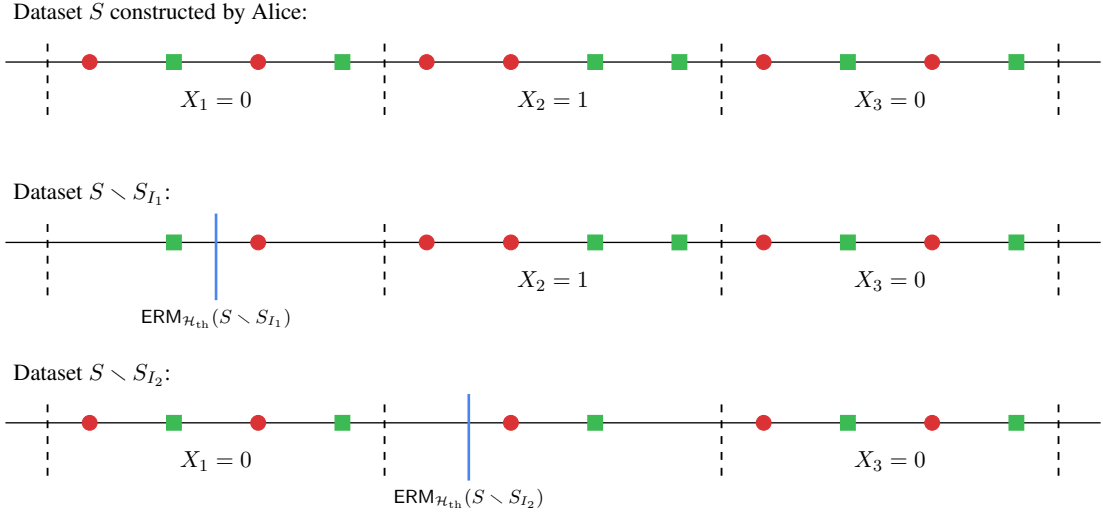
\begin{figure}
\centering
\begin{tikzpicture}[
    pt/.style = {minimum size=7pt,inner sep=0pt,outer sep=0pt},
    onept/.style = {pt,circle,draw=Gred,fill=Gred}, 
    zeropt/.style = {pt,draw=Ggreen,fill=Ggreen},
    ermline/.style = {Gblue,line width=1pt},
    scale=0.8, transform shape
]
\def \xgap{1.4}
\def \ygap{0.8}

\def \y {0}
\def \erm{14.5}

\node[right] at (-\xgap, \y + \ygap) {Dataset $S$ constructed by Alice:};
\draw[line width=0.5pt] (-\xgap, \y) -- (12*\xgap, \y);
\foreach \i in {0, 2, 4, 5, 8, 10}{%
    \node[onept] at (\i * \xgap, \y) {};
}
\foreach \i in {1, 3, 6, 7, 9, 11}{%
    \node[zeropt] at (\i * \xgap, \y) {}; 
}
\node at (1.5*\xgap, \y-0.8*\ygap) {$X_{1} = 0$};
\node at (5.5*\xgap, \y-0.8*\ygap) {$X_{2} = 1$};
\node at (9.5*\xgap, \y-0.8*\ygap) {$X_{3} = 0$};
\foreach \i in {0, 1, 2, 3} {
    \draw[dashed,line width=0.7pt]
    (\i*4*\xgap - 0.5*\xgap, \y-1.1*\ygap) --
    (\i*4*\xgap - 0.5*\xgap, \y+0.45*\ygap);
}

\def \y {-3}
\def \erm{1.5}

\node[right] at (-\xgap, \y + \ygap) {Dataset $S \smallsetminus S_{I_1}$:};
\draw[line width=0.5pt] (-\xgap, \y) -- (12*\xgap, \y);
\foreach \i in {2, 4, 5, 8, 10} {
    \node[onept] at (\i * \xgap, \y) {};
}
\foreach \i in {1, 6, 7, 9, 11} {
    \node[zeropt] at (\i * \xgap, \y) {};
}
\node at (5.5*\xgap, \y-0.8*\ygap) {$X_{2} = 1$};
\node at (9.5*\xgap, \y-0.8*\ygap) {$X_{3} = 0$};
\foreach \i in {0, 1, 2, 3} {
    \draw[dashed,line width=0.7pt]
    (\i*4*\xgap - 0.5*\xgap, \y-1.1*\ygap) --
    (\i*4*\xgap - 0.5*\xgap, \y+0.45*\ygap);
} 

\node[below] (ERM) at (\erm*\xgap, \y-1.2*\ygap) {\footnotesize $\ERM_{\Hth}(S \smallsetminus S_{I_1})$};
\draw[ermline] (\erm*\xgap, \y+0.6*\ygap) -- (ERM);

\def \y {-6}
\def \erm{4.5}

\node[right] at (-\xgap, \y + \ygap) {Dataset $S \smallsetminus S_{I_2}$:};
\draw[line width=0.5pt] (-\xgap, \y) -- (12*\xgap, \y);
\foreach \i in {0, 2, 5, 8, 10} {
    \node[onept] at (\i * \xgap, \y) {};
}
\foreach \i in {1, 3, 6, 9, 11} {
    \node[zeropt] at (\i * \xgap, \y) {};
}
\node at (1.5*\xgap, \y-0.8*\ygap) {$X_{1} = 0$};
\node at (9.5*\xgap, \y-0.8*\ygap) {$X_{3} = 0$};
\foreach \i in {0, 1, 2, 3} {
    \draw[dashed,line width=0.7pt]
    (\i*4*\xgap - 0.5*\xgap, \y-1.1*\ygap) --
    (\i*4*\xgap - 0.5*\xgap, \y+0.45*\ygap);
}

\node[below] (ERM) at (\erm*\xgap, \y-1.2*\ygap) {\footnotesize $\ERM_{\Hth}(S \smallsetminus S_{I_2})$};
\draw[ermline] (\erm*\xgap, \y+0.6*\ygap) -- (ERM);

\end{tikzpicture}
\caption{Illustration of lower bound reduction in \theoremref{thm:thresh-separation}, with red circles for label $1$ and green squares for label $0$.}
\label{fig:agnostic-thresh-central-lb}
\end{figure}

\paragraph{Alice:} On input $X \in \bit^m$
\begin{itemize}
\item Construct a dataset $S$ as follows:
\begin{itemize} 
\item For each $i \in [m]$,
    \begin{itemize}
        \item if $X_i = 0$ include examples $(4i + 1, 1)$, $(4i+2, 0)$, $(4i+3, 1)$, $(4i + 4, 0)$.
        \item if $X_i = 1$ include examples $(4i + 1, 1)$, $(4i+2, 1)$, $(4i+3, 0)$, $(4i + 4, 0)$.
    \end{itemize} 
\end{itemize} 
\item Let $(h, \aux) \gets \Learn(S)$ and send $\aux$ to Bob.
\end{itemize} 

\paragraph{Bob:} On receiving $\aux$ from Alice,  construct $Y \in \bit^m$ as follows:
\begin{itemize}
\item For each $i \in [m]$:
    \begin{itemize}
    \item Let $h^{(i)} \gets \Unlearn(S_{I_i}, \aux)$ for $S_{I_i} := ((4i + 1, 1), (4i+4, 0))$.
    \item If $h^{(i)} = h_{>4i+2}$, set $Y_i = 0$. Otherwise, set $Y_i = 1$.\\
    \end{itemize}
\end{itemize}

\noindent If $X_i = 0$, then after unlearning $S_{I_i}$, the only $h \in \ERM_{\Hth}(S \smallsetminus S_{I_i})$ is $h_{>4i+2}$. On the other hand, if $X_i = 1$, then after unlearning $S_{I_i}$, $h_{>4i+2} \notin \ERM_{\Hth}(S \smallsetminus S_{I_i})$. This reduction is visualized in \pref{fig:agnostic-thresh-central-lb}.

Thus, Bob can perfectly recover Alice's input $X$. Hence, we require that $\aux$ contains at least $m$ bits. 

\subsection{Upper Bounds for TiLU Schemes}\label{app:1d-threshold-ticket-ub} 
We now describe a TiLU scheme for $\Hth$ with space complexity $(C_s = \log |\cX|, C_t = \poly(\log n, \log |\cX|))$ that is valid for all distributions. We first describe the key ideas before delving into the finer details. We will construct a $\Learn$ algorithm that would return the predictor $h_{>a} \in \ERM_{\Hth}(S)$ with the smallest possible $a$ (hence forth referred to as the ``minimal ERM'').

For simplicity we assume that $|\cX| = D = 2^d$ is a power of $2$ (the proof can easily be generalized to other values). We use the following notation below for all $p, q \in \cX = \{1, \ldots, D\}$ with $p < q$:
\begin{itemize} 
\item $S_{p,q}$ denotes dataset within $S$ with $x$-values in $\{p, \ldots, q\}$.
\item Let $a_{p,q}$ be the smallest value such that $h_{>a_{p,q}} \in \ERM_{\Hth}(S_{p,q})$ and $p - 1 \le a_{p,q} \le q$; note that such a value exists because all examples in $S_{p,q}$ have $x$ values between $p$ and $q$ due to which $\cL(h_{>a}; S_{p,q})$ is same for all $a \le p-1$ and similarly, the same for all $a \ge q$.
\item $\err_{p,q} := \cL(h_{>a_{p,q}}; S)$. Intuitively speaking, $\err_{p,q}$ is the smallest loss incurred by any function $h_{>a}$ over $S$ subject to $p-1 \le a \le q$.
\end{itemize}

\noindent Consider an interval $[p, q]$ such that no example in $\U$ has $x$-value in $[p, q]$. Now, if the minimal ERM of $S \smallsetminus \U$ happens to be of the form $h_{>a}$ for $p-1 \le a \le q$, then, the minimal ERM has to be $h_{>a_{p,q}}$. Moreover, we have $\cL(h_{>a_{p,q}}; S \smallsetminus \U) = \err_{p,q} - \cL(h_{>a_{p,q}}; \U)$. 

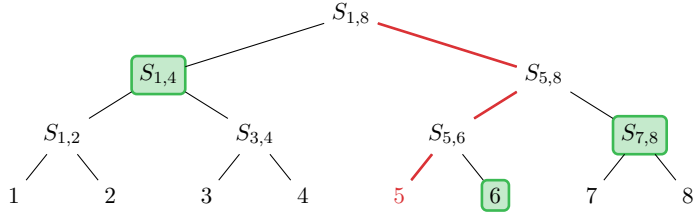
\begin{figure}
\centering
\begin{tikzpicture}[
  grow=down,
  level 1/.style={sibling distance=64mm},
  level 2/.style={sibling distance=32mm},
  level 3/.style={sibling distance=16mm},
  level 4/.style={sibling distance=8mm},
  level distance=10mm,
  highlight/.style={rectangle,rounded corners=2pt,draw=Ggreen,fill=Ggreen!30,line width=1pt},
  every node/.style = {outer sep=0pt},
  scale=0.8, transform shape
]
\node (1) {$S_{1,8}$}
  child {
    node[highlight] {$S_{1,4}$}
    child {
      node {$S_{1,2}$}
      child {
        node {$1$}
      }
      child {
        node {$2$}
      }
    }
    child {
      node {$S_{3,4}$}
      child {
        node {$3$}
      }
      child {
        node {$4$}
      }
    }
  }
  child {
    node (3) {$S_{5,8}$}
    child {
      node (6) {$S_{5,6}$}
      child {
        node[Gred] (12) {$5$}
      }
      child {
        node[highlight] {$6$}
      }
    }
    child {
      node[highlight] {$S_{7,8}$}
      child {
        node {$7$}
      }
      child {
        node {$8$}
      }
    }
  }; 
\path[Gred, line width=1pt]
(1) edge (3)
(3) edge (6) 
(6) edge (12); 
\end{tikzpicture}
\caption{Illustration of TiLU scheme underlying the proof of~\theoremref{thm:thresh-separation}(b) for $|\cX| = 8$. The ticket $t_i$ for all examples of the form $(x_i=5, y_i)$ are the predictors $h_{>a_{6}}$, $h_{>a_{7,8}}$, $h_{>a_{1,4}}$ as well as $\err_{6}$, $\err_{7,8}$ and $\err_{1,4}$.} 
\label{fig:agnostic-thresh}
\end{figure}

\paragraph{\boldmath $\Learn$.} 
Consider a full binary tree of depth $d$ with leaf $x$ corresponding to a possible input value $x \in \cX$. 
For each internal node $v$, let $p_v$ be the smallest leaf index under $v$, and similarly let $q_v$ be the largest leaf index under $v$. For each example $(x_i, y_i)$, let $v_1, \dots, v_{d - 1}$ be the nodes on the path from the root to leaf $x_i$, and for each $j \in \crl{2,\ldots, d}$ let $\tilde{v}_j$ be the child of $v_{j-1}$ and sibling of $v_j$. \pref{fig:agnostic-thresh} shows an example.
Let the ticket corresponding to example $(x_i, y_i)$ be given as
\[
t_i \defeq \prn{(a_{p_{v_2}, q_{v_2}}, \err_{p_{v_2}, q_{v_2}}), \ldots, (a_{p_{v_d}, q_{v_d}}, \err_{p_{v_d}, q_{v_d}})}
\]
It is immediate to see that the number of bits in $t_i$ is $d(d + \log n)$. Define $\Learn(S)$ to return $h = h_{>a_{1,D}}$, $\aux=h$, and tickets $t_i$ as specified above.

\paragraph{\boldmath $\Unlearn$.} If $\U$ is empty, then simply return $h$ (present in $\aux$). Given a non-empty dataset $\U$ indexed by $I\subseteq [n]$, we can use the tickets to construct a partition of $[1, D]$ into disjoint intervals such that either each interval is a singleton $\{p\}$ such that some example in $\U$ has $x$-value equal to $p$, or is of the form $[p,q]$ such that no example in $\U$ has an $x$-value in $[p,q]$. As argued above, we know that the minimal ERM for $S \smallsetminus \U$ must be of the form $h_{>a_{p,q}}$ for one of these parts as recovered above. Moreover, it is possible to compute the loss of each such predictor over $S$ as $\cL(h_{>a_{p,q}}; S \smallsetminus \U) = \err_{p,q} - \cL(h_{>a_{p,q}}; \U)$. Thus, we can choose the minimal ERM for $S \smallsetminus \U$ by choosing the predictor $h_{>a}$ with smallest empirical loss over $S \smallsetminus \U$ and the smallest possible value of $a$ among the candidates above.

\section{Missing Proofs from \texorpdfstring{\pref{sec:sharper}}{Section~\ref{sec:sharper}}} \label{app:sharper-proofs}

\subsection{Point Functions}

\begin{proofof}[\theoremref{thm:sharper-bounds-main}\ref{item:point-ctz}] 
In the following, we provide the methods \(\Learn\) and \(\Unlearn\). Our algorithm will use the TiLU scheme $(\tLearn, \tUnlearn)$ for CtZ from \theoremref{thm:ctz-final-upper-bound}. One specific property of this scheme we will use is that $\tUnlearn$ does not need $\aux$ in the case where the unlearning set is non-empty.

\paragraph{\boldmath \(\Learn\).} We will separately define the hypothesis $h$, the auxiliary information $\aux$, and the ticket $(t_i)_{i \in [n]}$. Recall that $S$ is the input set of training examples.
\begin{itemize}
\item The hypothesis $h$ is defined as follows:
\begin{align*}
h = 
\begin{cases} 
h_a & \text{ if } \exists a \in \cX, (a, 1) \in S, \\ 
h_{\min\{b ~\mid~ (b, 0) \notin S\}} & \text{ otherwise.}
\end{cases}
\end{align*}
\item The auxiliary information $\aux$ consists of three parts $\aux_1, \aux_2, \aux_3$. The first part $\aux_1$ is simply $h$. The second part $\aux_2$ is one bit and records which of the two cases we are in for $h$, i.e., $\aux_2 := \indicator{\exists a \in \cX, (a, 1) \in S}$. The last part $\aux_3$ records $\min\{b \mid (b, 0) \notin S\}$.
\item As for the tickets, for every $a \in \cX$, let $S^{(a)}$ denote $\{(x_i, y_i) \mid i \in [n], x_i = a\}$. For every $a \in \cX$, we run $\tLearn$ on $S^{(a)}$ and let the tickets for the examples in $S^{(a)}$ be as the output of $\tLearn$.
\end{itemize}

\paragraph{\boldmath \(\Unlearn\).} $\Unlearn$ proceeds as follows, based on if $\aux_2 = 1$. Here, for every $a \in \cX$, let $S^{(a)}_I$ denote $\{(x_i, y_i) \mid i \in I, x_i = a\}$.
\begin{itemize}
\item Case I: $\aux_2 = 1$. Let $\aux_1 = h_a$. In this case, start by checking if $S^{(a)}_I = \emptyset$. 
\begin{itemize}
\item If $S^{(a)}_I = \emptyset$, then output $h_a$.
\item If $S^{(a)}_I \ne \emptyset$, then run $\tUnlearn$ on the tickets corresponding to $S^{(a)}_I$ to determine if $S^{(a)}_I = S^{(a)}$.
\begin{itemize}
\item If $S^{(a)}_I \ne S^{(a)}$ (i.e., $\tUnlearn$ outputs $\top$), then output $h_a$.
\item If $S^{(a)}_I = S^{(a)}$ (i.e.,  $\tUnlearn$ outputs $\bot$), then proceed to Case II.
\end{itemize}
\end{itemize}
\item Case II: $\aux_2 = 0$. In this case, start with $b = \aux_3$. Then, for $j = b - 1, \dots, 1$, do the following:
\begin{itemize}
\item If $S^{(j)}_I = \emptyset$, then skip to the next $j$.
\item Otherwise, if $S^{(j)}_I \ne \emptyset$ , then run $\tUnlearn$ on the tickets corresponding to $S^{(j)}_I$ to determine if $S^{(j)}_I = S^{(j)}$. If $S^{(j)}_I = S^{(j)}$, then update $b \gets j$.
\end{itemize}
Finally, output $h_b$.
\end{itemize}

To verify that this is a valid learning scheme, note that for $a$ and $j$'s used in the algorithm, we know that $S^{(a)}$ and $S^{(j)}$ are not empty. Therefore, by considering the case  $S^{(a)}_I, S^{(j)}_I \ne \emptyset$ and uses $\tUnlearn$, we can determine whether $S^{(a)} = S^{(a)}_I$ or $S^{(j)} = S^{(j)}_I$. Due to this, our procedure correctly updates (i) if there is any 1-labeled example remaining and (ii) the minimum $b$ whose $(b, 0)$ does not appear in the dataset. From this, it is not hard to see that $(\Learn, \Unlearn)$ is a valid ticketed LU scheme. The claimed space complexity the follows immediately from that of \theoremref{thm:ctz-final-upper-bound}.
\end{proofof}

\subsection{\boldmath Product of \texorpdfstring{$d$}{d} Thresholds}

\subsubsection{Minimum Value Primitive}

\begin{theorem} \label{thm:min-scheme}
There is a TiLU scheme for $\minval$ with space complexity $(O(\log|\cX|), O(\log|\cX| + \log \alpha^{-1}(n)))$.
\end{theorem}

\begin{proof}
For each $a \in \cX$, we write $\succe_S(a)$ to denote the maximum value in $S$ that is larger than $a$; if such a value does not exist, then let $\succe_S(a) = \bot$. Similarly, we use the convention that $\min(\emptyset) = \bot$. Similar to before, we use $S^{(a)}$ (resp. $S^{(a)}_I$) to denote $\{x_i \mid i \in [n], x_i = a\}$ (resp. $\{x_i \mid i \in I, x_i = a\}$), and let $(\tLearn, \tUnlearn)$ be the ticketed LU scheme for CtZ from \theoremref{thm:ctz-final-upper-bound}.  We now describe our algorithms.

 \paragraph{\boldmath \(\Learn\).} First, we let $h = \aux = \min(S)$. Each ticket $t_i$ consists of two components $t_i^{(1)}, t_i^{(2)}$. The second component $t_i^{(2)}$ is simply set to $\succe_S(x_i)$. As for the first component, it is computed as follows: for every $a \in \cX$, we run $\tLearn(S^{(a)})$ and assign the tickets back to samples in $S^{(a)}$ as a first component.

 \paragraph{\boldmath \(\Unlearn\).} The algorithm works as follows. Start by letting $b$ be such that $\aux = h_b$. Then, do the following:
\begin{itemize}
\item If $b = \bot$ or $S^{(b)}_I = \emptyset$, then output $b$ and stop.
\item Otherwise, if $S^{(b)}_I \ne \emptyset$, then run $\tUnlearn$ on the second component of the tickets corresponding to $S^{(b)}_I$ to determine whether $S^{(b)}_I = S^{(b)}$. If $S^{(b)}_I \ne S^{(b)}$, then output $b$ and stop. Otherwise, if $S^{(b)}_I = S^{(b)}$, then update $b \gets \succe_S(b)$ where $\succe_S(b)$ is taken from the first component of the tickets, and repeat this step.
\end{itemize}

To see the correctness, once again note that we are simply checking in each iteration whether $b$ still belong to $S \setminus S_I$ and otherwise we let $b \gets \succe_S(b)$. The correctness of the check follows similar to the previous proof. Moreover, given the check is correct, it is clear that the entire algorithm is correct.
\end{proof}

\subsubsection{From Minimum Value to Product of \texorpdfstring{$d$}{d} Thresholds}

\begin{proofof}[\theoremref{thm:sharper-bounds-main}\ref{item:prod-threshold-ctz}]
Let $(\tLearn, \tUnlearn)$ denote the TiLU scheme for $\minval$ from \theoremref{thm:min-scheme}. The algorithms work as follows.

\paragraph{\boldmath \(\Learn\).} Let $S^{(+)}$ denote the set $\{(x_i, y_i) \mid i \in [n], y_i = 1\}$ and for each $j \in [d]$, let $S^{(+)}_j$ denote $\{(x_i)_j \mid (x_i, y_i) \in S^{(+)}\}$. The output hypothesis is $h_a$ where 
\begin{align*}
a_j = 
\begin{cases}
m_j & \text{ if } S^{(+)} = \emptyset, \\
\min(S^{(+)}_j) - 1 & \text{ otherwise.}
\end{cases}
\end{align*}

As for the ticket, if $y_i = 0$, then the ticket $t_i$ is empty. For the 1-labeled samples, their ticket $t_i$ is the concatenation of the corresponding tickets from $\tUnlearn(S^{(+)}_1), \dots, \tUnlearn(S^{(+)}_d)$. Similarly, the auxiliary information is the concatenation of the auxiliary information from these $\tUnlearn$ executions.

\paragraph{\boldmath \(\Unlearn\).} We use $\tUnlearn$ to find $b_j = \min(S^{(+)}_j \setminus (S_I)^{(+)}_j)$. We then let
\begin{align*}
a_j = 
\begin{cases}
m_j & \text{ if } b_j = \bot, \\
b_j - 1 & \text{ otherwise.}
\end{cases}
\end{align*}
Finally, output $h_{>\ba}$.

It is simple to see that the learning algorithm is correct: by definition, all 1-labeled examples are correctly labeled by $h_{>\ba}$. Furthermore, if the input training set is realizable by $h_{>\ba^*}$, then we must have $a^*_j \leq a_j$ for all $j \in [d]$; this implies that $h_{>\ba}$ also label all 0-labeled examples correctly. The correctness of the unlearning algorithm then immediately follows from the correctness of $(\tLearn, \tUnlearn)$ for \minval.

Since the ticket is a concatenation of the $d$ tickets for \minval, we can use the bound in \theoremref{thm:min-scheme}. This implies that the total ticket size is $O\left(d \cdot \left(\log m + \log \alpha^{-1}(n)\right)\right) = O(\log|\cX| + d \cdot \log \alpha^{-1}(n)))$ as desired. A similar calculation shows that the auxiliary information has size $O(\log |\cX|)$. 
\end{proofof}

\subsection{Thresholds}

\begin{proofof}[\theoremref{thm:sharper-bounds-main}\ref{item:threshold-ctz}]
Let $(\tLearn, \tUnlearn)$ be the TiLU scheme for CtZ from \theoremref{thm:ctz-final-upper-bound}. For any interval $[\ell, r]$, we write $S^{[\ell, r]}$ (resp. $S_I^{[\ell, r]}$) to denote $\{(x_i, y_i) \in S \mid x_i \in [\ell, r]\}$ (resp. $\{(x_i, y_i) \in S_I \mid x_i \in [\ell, r]\}$).

\paragraph{\boldmath $\Learn$.}
Let $a_0, \dots, a_i$ be the same as in the proof of \theoremref{thm:1d-thresh-central-ub}. Let $j_0, \dots, j_i$ be such that $a_{j_0} < \cdots < a_{j_i}$ Then, for each $k = 0, \dots, i - 1$, we run $\tLearn(S^{[a_{j_k} + 1, a_{j_{k+1}}]})$ to get $\aux_k$ and the tickets, which we assign to the corresponding examples. The final auxiliary information is then the concatenation of $\aux_1, \dots, \aux_{i - 1}$.

\paragraph{\boldmath $\Unlearn$.}
Similar to the proof of \theoremref{thm:1d-thresh-central-ub}, the only task for us is to determine whether $\cL(h_{>a_j}; S \smallsetminus \U) = 0$ for each $j = 0, \dots, i - 1$. To do this, observe that $\cL(h_{>a_j}; S \smallsetminus \U) = 0$ iff $S \smallsetminus \U$ contains no point in $[a_j + 1, a_i] \cup [a_i + 1, a_j]$. We can check this by first running $\tUnlearn$ on each of $U^{[a_{j_k} + 1, a_{j_{k+1}}]}$ with auxiliary information $\aux_k$ to check whether $S \smallsetminus \U$ contains any point in $[a_{j_k} +  1, a_{j_{k+1}}]$. This information is sufficient to determine whether $S \smallsetminus \U$ contains any point in $[a_j + 1, a_i] \cup [a_i + 1, a_j]$.
\end{proofof}

\section{Missing Proofs from \texorpdfstring{\pref{sec:lb}}{Section~\ref{sec:lb}}}
\begin{proofof}[\lemmaref{lem:sperner-to-ctz-lb}]
Suppose for the sake of contradiction that there is a TiLU scheme $(\Learn, \Unlearn)$ for CtZ with space complexity $(C_s = c_s, C_t = c_t)$ where both $c_s, c_t \in \N$ are absolute constants. %

For every $p \in \{0, 1\}^{c_s}$, let $A^p$ denote the set of $a \in \N$ such that, if we feed in $a$ elements to $\Learn$ as input, the returned $\aux$ is equal to $p$. Since $\bigcup_{p \in \crl{0, 1}^{c_s}} A^p = \N \cup \{0\}$, there must exist $p^*$ such that $A^{p^*}$ is infinite. Let $a_1 < a_2 < \cdots$ be the elements of $A^{p^*}$ sorted in increasing order. 

Next, for every $i \in \N$, let $Q_i$ denote the set of all $a_i$ tickets produced by $\Learn$ when the input contains $a_i$ elements. Since $\{Q_i\}_{i \in \N}$ is an $\ba$-size-indexing family with alphabet size $\sigma = 2^{c_t}$, \theoremref{thm:sperner-lb} implies that it cannot be a Sperner family. In other words, there must exists $k < \ell$ such that $Q_k \subseteq Q_\ell$. Now, consider the $\Unlearn$ algorithm when it receives $\aux = p^*$ and tickets being $Q_k$. If it answers $\bot$, then this is incorrect for the case where we start with $a_\ell$ elements and removes $a_k$ elements with the tickets corresponding to $Q_k$. However, if it answers $\top$, then it is also incorrect for the case where we start with $a_k$ elements and remove everything. This is a contradiction.
\end{proofof} 

\begin{proofof}[\theoremref{thm:count-to-zero-lb-red}]
Let $(\Learn, \Unlearn)$ be any TiLU scheme for $\cH$. Let $h_1$ denote the predictor output by $\Learn(\emptyset)$ and $h_2$ denote any other hypothesis in $\cH$, and let $x$ be such that $h_1(x) \ne h_2(x)$. 

Our TiLU scheme for CtZ $(\tLearn, \tUnlearn)$ works as follows. $\tLearn(S)$ returns the same auxiliary information and tickets as $\Learn((x, h_2(x)), \dots, (x, h_2(x)))$, where $(x, h_2(x))$ is repeated $|S|$ times. $\tUnlearn(S_I, \aux, (t_i)_{i \in I})$ first runs $\Unlearn(((x, h_2(x)), \dots, (x, h_2(x))), \aux, (t_i)_{i \in I})$ where $(x, h_2(x))$ is repeated $|S_I|$ times, to get a predictor $h'$. Then, it outputs $\top$ iff $h'(x) = h_2(x)$. 

It is obvious that the space complexity of $(\tLearn, \tUnlearn)$ is the same as $(\Learn, \Unlearn)$. For the correctness, note that if $S_I = S$, then $\Unlearn$ outputs $h_1$ (by our choice of $h_1$) and thus $\tUnlearn$ correctly outputs $\bot$. On the other hand, if $S_I \ne S$, then at least one copy of $(x, h_2(x))$ remains and therefore we must have $h'(x) = h_2(x)$. Hence, in this case, $\tUnlearn$ outputs $\top$ as desired. 
\end{proofof}

\section{Realizability Testing} \label{app:realizability_testing} 

In this section, we describe the \textit{realizability testing} problem, where the goal is to test whether a given dataset \(S\) is \(\cH\)-realizable, i.e.,~ there exists  \(\hstar \in \cH\) s.t.~\(\hstar(x) = y\) for all \((x, y) \in S\). We extend the terminologies of \definitionref{def:lu-ticket} except that $\Learn$ and $\Unlearn$ now return $\perp$ (when $S$ is realizable), or $\top$ (when $S$ is not realizable by \(\cH\)), instead of the hypotheses $h, h'$. 

The easier scenario is when \(S\) is \(\cH\)-realizable, in which case, for any \(\U \subseteq S\), the remaining dataset \(S \setminus \U\) is also \(\cH\)-realizable. However,  when \(S\) is not \(\cH\)-realizable, removing   \(\U\) may make \(S \setminus \U\) to be \(\cH\)-realizable, and this is the challenging scenario for realizability testing.  

Similar to the rest of the results in the paper, we show a separation between central LU and TiLU schemes for realizability testing, by considering the class of 1D thresholds. 

\begin{theorem}\label{thm:thresh-realizability-separation} 
For realizability testing for the class $\Hth$ of 1D thresholds,
\begin{enumerate}[label=\((\alph*)\)]   
\item Any LU scheme that is valid for all datasets must have space complexity $\Omega(\min\crl{n, |\cX|})$. 
\item There exists a TiLU scheme valid for all datasets with space complexity $(C_s = O(\log|\cX|), C_t = O(\log|\cX| + \log \alpha^{-1}(n)))$. 
\end{enumerate}
\end{theorem} 

Before presenting the proof of \pref{thm:thresh-realizability-separation}, we first describe a TiLU scheme for \emph{MAXimum VALue} (\maxval) problem, which will be a useful primitive for developing TiLU schemes for realizability testing in 1D-thresholds. In the \(\maxval\) problem, we are given a set $S \subseteq \cX$ where $\cX \subseteq \bbR$ is the domain and the goal is to output $\max(S)$ if $S \ne \emptyset$ and $\perp$ if $S = \emptyset$. 

\begin{theorem} \label{thm:max-scheme}
There is a TiLU scheme for $\maxval$ with space complexity $(O(\log|\cX|), O(\log|\cX| + \log \alpha^{-1}(n)))$. 
\end{theorem}
\begin{proof} The scheme follows similar to the TiLU scheme for \minval problem given in \pref{thm:min-scheme}, and is skipped for conciseness. 
\end{proof}

\begin{proofof}[\pref{thm:thresh-realizability-separation}.]  
(a) Suppose for contradiction, we have a LU scheme in the central model for $\Hth$-realizability testing, consisting of $\Learn$ and $\Unlearn$ methods, that is valid for all datasets. Suppose $|\cX| \ge n$, and let \( m  = n/2\). We will construct a two-party one-way communication protocol, with communication complexity being the space complexity of the LU scheme, and where Bob can entirely reconstruct Alice's input $X \in \bit^m$ (with probability $1$). 

\begin{enumerate}[leftmargin=8mm, label=\(\hspace{0.2in}\)] 
	\item \textbf{Alice:} On input $X \in \bit^{m}$, 
\begin{itemize}
\item Construct a dataset $S$ as follows: 
For each $i \in [m]$, include examples $(2i + 1, 0)$ and $(2i + X_i + 1, 1)$
\item Let $(h, \aux) \gets \Learn(S)$ and send $\aux$ to Bob.
\end{itemize} 

\item \textbf{Bob:} On receiving $\aux$ from Alice,  construct $Y \in \bit^{m}$ as follows: 
\begin{itemize} 
\item For each $i \in [m]$: 
    \begin{itemize}
    \item Let $v^{(i)} \gets \Unlearn(S_{I_i}, \aux)$ for $S_{I_i} := \{(2i + 1, 0)\}_{i \in [2m] \setminus \{i\}} \cup \{(2j + Y_j + 1, 1)\}_{j \in [i - 1]}$. %
    \item If $v^{(i)} = \top$ (non-realizable), set $Y_i = 0$. %
   Else if $v^{(i)} = \perp$ (realizable), set $Y_i = 1$. %
    \end{itemize} 
    \item Return \(Y\). 
\end{itemize} 
\end{enumerate}  

\noindent The proof proceeds by arguing that Bob can successfully recover \(Y = X\), which implies that \(\aux\) must contain at least \(m\) bits. %
To prove the correctness of the construction of $Y$, we can use induction on $i$. Suppose that we have constructed $Y_1, \dots, Y_{i - 1}$ correctly. Then, in iteration \(i\), Bob makes an unlearning request \(S_{I_i}\) to remove all \(0\)-labeled samples except \((2 i + 1, 0)\), and all \(1\)-labeled samples with \(x\) values smaller than \(2i + 1\)
There are two cases:
\begin{enumerate}[label=\(\bullet\)] 
\item $X_i = 0$: In this case, \(S \setminus S_{I_i}\) contains $(2i + 1, 1)$ and is not $\Hth$-realizable, meaning that we set $Y_i = 0$.
\item $X_i = 1$: In this case, \(S \setminus S_{I_i}\) does not contain $(2i + 1, 1)$ and thus is $\Hth$-realizable, leading to $Y_i = 1$.
\end{enumerate}
Therefore, $Y_i = X_i$ in both cases as desired. The lower bound follows since \(n  = 2m\). 

\medskip\noindent
(b) To obtain a TiLU scheme, the key idea is that in order to check whether $S$ is realizable via a threshold, we only need to verify if the largest \(0\)-labeled sample is strictly smaller than the smallest \(1\)-labeled sample. The smallest \(1\)-labeled sample can be computed using the \minval primitive in \pref{thm:min-scheme}, and the largest \(0\)-labeled sample can be computed using the \maxval primitive in \pref{thm:max-scheme}, both of which have a  TiLU scheme with space complexity $(O(\log|\cX|), O(\log|\cX| + \log \alpha^{-1}(n)))$.

With that in mind, we define $\Learn$ and $\Unlearn$ as follows.
\paragraph{\boldmath $\Learn$.} On input \(S\), 
\begin{itemize}
    \item Compute \(S^0\) as the set of all \(0\)-labeled samples in \(S\), and \(S^1 = S \setminus S^0\). 
    \item Implement \maxval on \(S^0\), let \(x^0 = \max(S^0)\) and compute the corresponding tickets for samples in \(S^0\). 
    \item Implement \minval on \(S^1\), let \(x^1 = \min(S^1)\) and compute the corresponding tickets for samples in \(S^1\). 
    \item If either \(x^0\) or \(x^1\) is \(\bot\), or if \(x^0 < x^1\), set \(\aux = (\bot, x^0, x^1)\), and return. 
    \item Else, set    \(\aux = (\top, x^0, x^1)\) and return. 
\end{itemize}

\paragraph{\boldmath $\Learn$.} On input \(\aux = (a, x^0, x^1)\), unlearning requests \(S_I\) and the corresponding tickets \(\crl{t_z}_{z \in S_I}\), return \(a\) if \(S_I\) is empty. Otherwise, 
\begin{itemize}
    \item If \(a = \bot\), return \(\bot\). 
    \item Else, compute \(S_I^0\) and \(S_I^1\) as the set of \(0\)-labeled and \(1\)-labeled delete requests respectively. Similarly, define the tickets \(T_I^0\) and \(T_I^1\) as the tickets for the \(0\)-labeled and \(1\)-labeled samples in \(I\) respectively. 
    \item Compute the max \(0\)-labeled sample in remaining dataset via \(x^0 = \maxval.\Unlearn(S_I^0, x^0, T_I^0)\). 
        \item Compute the min\(1\)-labeled sample in remaining dataset via \(x^1 = \minval.\Unlearn(S_I^1, x^1, T_I^1)\).  
        \item If \(x^0 = \bot\) or \(x^1 = \bot\), return \(\bot\). 
        \item Else if \(x^0 < x^1\), return \(\bot\). Else, return \(\top\). 
\end{itemize}

Since, the \(\minval\) and \(\maxval\) primitives compute the minimum $1$-labeled \(x^1\), and the maximum \(0\)-labeled samples \(x^0\). Additionally, since realizability testing w.r.t.~\(\Hth\) is equivalent to checking whether \(x^0 < x^1\), one can verify that the above algorithm is correct. Since  the \(\minval\) and \(\maxval\) primitives can be executed using a TiLU scheme with space complexity $(O(\log|\cX|), O(\log|\cX| + \log \alpha^{-1}(n)))$, the above TiLu scheme for realizability testing also has space complexity $(O(\log|\cX|), O(\log|\cX| + \log \alpha^{-1}(n)))$. 

\end{proofof}

\end{document}